\documentclass[aps,prl,twocolumn]{revtex4-1}
\usepackage[utf8]{inputenc}
\usepackage[T1]{fontenc}
\usepackage{amsthm}
\usepackage{thmtools}
\usepackage{thm-restate}
\usepackage{dsfont}
\usepackage{bm,bbm}
\usepackage{graphicx}
\usepackage{amsmath}
\usepackage{amssymb}
\usepackage{latexsym}
\usepackage{graphicx}
\usepackage[caption=false]{subfig}
\usepackage{url}
\usepackage{epstopdf}
\usepackage{mathtools}
\usepackage[colorlinks=true,urlcolor=blue,citecolor=blue]{hyperref}
\usepackage[paperwidth=210mm,paperheight=297mm,centering,hmargin=2.8cm,vmargin=2.5cm]{geometry}

\newcommand{\bra}[1]{\left\langle #1\right|}
\newcommand{\ket}[1]{\left| #1\right\rangle}
\newcommand{\ketbra}[1]{\left|{#1}\rangle\!\langle{#1}\right|}

\newcommand{\tr}{\mathrm{tr}}
\newcommand{\id}{\mathrm{id}}
\renewcommand{\P}{\mathbb P}
\newcommand{\E}{\mathbf E}
\newcommand{\EV}{\mathbb E}
\newcommand{\Q}{\mathbb Q}
\newcommand{\A}{\mathcal A}
\newcommand{\B}{\mathcal B}
\newcommand{\X}{\mathcal X}
\newcommand{\Y}{\mathcal Y}
\newcommand{\M}{\mathcal M}
\newcommand{\sym}{\mathrm{sym}}

\newtheorem{definition}{Definition}
\newtheorem{theorem}{Theorem}
\newtheorem{lemma}{Lemma}

\begin{document}
\newcommand{\app}{Supplemental Material, which includes Refs.~\cite{bengtsson2006geometry,Christandl2007,Christandl2009a}}.

\title{Inductive supervised quantum learning}

\author{Alex Monr\`{a}s$^1$, Gael Sent\'{i}s$^2$, Peter Wittek$^{3,4}$}
\affiliation{
$^1$F\'isica Te\`orica: Informaci\'o i Fen\`omens Qu\`antics, Universitat Aut\`onoma de Barcelona, ES-08193 Bellaterra (Barcelona), Spain\\
$^2$Departamento de F\'isica Te\'orica e Historia de la Ciencia, Universidad del Pa\'{i}s Vasco UPV/EHU, E-48080 Bilbao, Spain\\
$^3$Naturwissenschaftlich-Technische Fakult\"at, Universit\"at Siegen, 57068 Siegen, Germany
$^4$ICFO-The Institute of Photonic Sciences, Castelldefels, E-08860 Spain\\
$^5$University of Bor{\aa}s, Bor{\aa}s, S-50190 Sweden
}

\begin{abstract} In supervised learning, an inductive learning algorithm extracts general rules from observed training instances, then the rules are applied to test instances.
We show that this splitting of training and application arises naturally, in the classical setting, from a simple independence requirement with a physical interpretation of being non-signalling.
Thus, two seemingly different definitions of inductive learning happen to coincide.
This follows from the properties of classical information that break down in the quantum setup.
We prove a quantum de Finetti theorem for quantum channels, which shows that in the quantum case, the equivalence holds in the asymptotic setting, that is, for large number of test instances.
This reveals a natural analogy between classical learning protocols and their quantum counterparts, justifying a similar treatment, and allowing to inquire about standard elements in computational learning theory, such as structural risk minimization and sample complexity.
\end{abstract}

\maketitle

Real-world problems often demand optimizing over massive amounts of data.
Machine learning algorithms are particularly well suited to deal with such daunting tasks: by mimicking a learning process, data is handled in a tractable way and approximately optimal solutions are inferred.
Quantum machine learning, an emergent line of research that introduces quantum resources in learning algorithms~\cite{wittek2014qml,Schuld2015introduction,adcock2015advances,biamonte2016quantum,arunchalam2017survey}, brings this pragmatic approach to quantum information processing, with a strong emphasis on speedup~\cite{aimeur2013quantumspeedup,rebentrost2014quantum,cai2015entanglement-based,lloyd2016quantumalgorithms}.
Quantum mechanics, however, also alters the limitations on what is physically possible in a classical learning setup, thus potentially changing the structure of learning algorithms at a fundamental level and opening a door for increasing performance.
In particular, handling quantum data collectively typically allows to outperform local approaches in many information processing tasks~\cite{Bennett1999,Giovannetti2001,Giovannetti2004,Niset2007,Sentis2014a,Sentis2016}.
Investigating the potential advantage of using quantum resources in learning algorithms crucially demands to establish the ultimate limits achievable within the framework of quantum machine learning.
This Letter tackles the question for general inductive supervised learning scenarios.

In machine learning, we are given a sample of a distribution called training instances~\cite{devroye1996probabilistic,hastie2008statisticallearning}. 
The training instances may have further fine structure, and we often think of them as pairs consisting of an input object and a matching output value or label: this is the scenario of supervised machine learning, where a classifying function is induced from the training instances and then used to assign labels to a number of unlabelled instances that we call test instances.
Not all forms of supervised learning are inductive: transductive learning refers to a problem in which labelled training instances are available, as well as unlabelled instances~\cite{gammerman1998learning}.
The task is to propagate the labels to the unlabelled ones, that is, we do not require inducing a function that we can use infinitely many times.
In this case, the geometry of both the labelled and unlabelled instances will influence the outcome.
At variance with supervised learning, in which the training occurs in a single step, reinforcement learning algorithms are trained on an instance basis with the possibility of changing the distribution by the subsequent querying, and the quantum generalization of the scenario has already been studied~\cite{dunjko2016quantumenhanced}.

In this Letter we develop a framework for inductive supervised quantum learning, that is, when training and test instances are given in a quantum form, and we contrast its structure with its classical analog.
We first show for the classical case that a natural independence requirement among test instances, i.e.
that the learning algorithm be non-signalling, induces the standard splitting of inductive learning algorithms into a training phase and a test phase.
We then prove that the same splitting holds asymptotically in the quantum case, despite having access to coherent collective quantum operations. 
In other words, we show that, in a fully quantum setting, the following three statements are equivalent in terms of performance in the asymptotic limit: (i) supervised learning algorithm learns a function which is applied to every test instance; (ii) supervised learning algorithm satisfies a non-signalling criterion; (iii) supervised learning scenario splits into separate training and test phases.

More formally, we derive a de Finetti theorem for non-signalling quantum channels and use it to prove that the performance of any quantum learning algorithm, under the restriction of being non-signalling, approaches that of a protocol that first measures the training instances and then infers labels on the test instances, in the limit of many test instances. Our result reveals a natural analogy between classical and quantum learning protocols, justifying a similar treatment that we have been taken for granted. Ultimately, the result provides a solid basis to generalize key concepts in statistical learning theory, such as structural risk minimization~\cite{vapnik1995nature}, to quantum scenarios.

\emph{Classical inductive learning}.
Consider a supervised learning problem characterized by an unknown joint probability distribution $P_{XY}$, where $X$ and $Y$ are random variables that model the test data and the label associated to it, respectively.
We denote its respective marginals by $P_X$ and $P_Y$.
We are given a finite set of i.i.d. unlabelled test instances $\{x_i\}_{i=1}^{n}$ and a set of correctly labelled examples called training set. 
The training set is generated by sampling the distribution $P_{XY}$, and we model it by the random variable $A=\{(X_1,Y_1),\ldots, (X_m,Y_m)\}$.
We are then set to solve the task of assigning a label $y_i$ to each test instance $x_i$, based on the information contained in the training set $A$.
We define a {\em learning protocol} that implements this task by a stochastic map $\P(y_1,\dots,y_n|A,x_1,\dots,x_n)$.

The natural figure of merit for assessing the performance of a learning protocol is the \emph{expected risk}, defined in terms of the conditional expected risk or average score per test instance $\E[\P|A]=
	\sum_{i}\mathbb E[s_{y_i,y'_i}\P(y_{1:n}|A,x_{1:n})]$, where $y'_i$ are the true labels, accessible to a referee for evaluation purposes, $s_{j,k}=1-\delta_{j,k}$, and we have introduced the short-hand notation $x_{1:n}=\{x_1,\ldots,x_n\}$ (likewise for $y$ and $y'$).
The expectation is in terms of variables $x_{1:n},y'_{1:n}$ over the distribution $P_{XY}$, i.e., for a generic function $g$, $\mathbb E[g(x,y,y')]=\sum_{x,y,y'}g(x,y,y')P_{XY}(x,y')$.
The expected risk is then defined as the average conditional expected risk over realizations $a$ of the training set, i.e.
$\E[\P]=\sum_a p_A(a)\E[\P|a]$.

It is convenient to define the marginal maps of $\P$:
\begin{align}
	\P_i(y_i|A,x_{1:n})\equiv
	\sum_{y_{1:i-1},y_{i+1:n}}
	\P(y_{1:n}|A,x_{1:n})\,.
\end{align}
We call a learning protocol \emph{inductive} if it satisfies the condition
\begin{align}  \label{eq:nonsignalling}
	\P_i(y_i|&A,x_{1:i-1},x_i,x_{i+1:n})=\\
\nonumber
	&\P_i(y_i|A,x'_{1:i-1},x_i,x'_{i+1:n})\,, \forall
	x'_{1:i-1}\,,\,x'_{i+1:n},
\end{align}
for all $i$,
namely, that $\P_i(y_i|A,x_{1:n})$ is actually independent of all the $X$ random variables but $X_i$, for all $i$. 
Eq.~\eqref{eq:nonsignalling} can be interpreted as a \emph{non-signalling} condition among the test instances as far as the learning protocol is concerned.
Note, however, that each marginal map $\P_i$ is still affected by the training set $A$.
This definition encompasses the standard assumption of inductive learning, where a classifier $f$ is extracted from the training set, and only $f$ determines the label to be assigned to each test instance.

In contraposition, consider a transductive learning scenario, where the topology of all of the unlabelled instances can have an impact on the assignment of any of the labels.
The independence condition in Eq.~(\ref{eq:nonsignalling}) is thus violated, hence a transductive protocol is potentially signalling.

The following lemma pinpoints the feature of classical inductive learning that is relevant for our goal.
In the next section we explore its extension to quantum settings.

\begin{lemma}\label{lem:inductive_classical}For every inductive learning protocol $\P$ that assigns labels $y_i$ to test instances $x_i$, there exists a set $F$ of classifying functions $f:X\rightarrow Y$ and stochastic maps $T(f|A)$, $Q(y|x,f)=\delta_{y,f(x)}$ such that the inductive protocol $\tilde \P$
\begin{align}
	\tilde \P(y_{1:n}|A,x_{1:n})=\sum_f  \left[\prod_{i=1}^n Q(y_i|x_i, f)\right] T(f|A)\nonumber
\end{align}
has expected risk $\E[\P|A]= \E[{\tilde \P}|A]$ for all $A$.
\end{lemma}
The proof can be found in~\footnote{See \app.}.
In words, this lemma shows that every conceivable inductive learning protocol can be regarded, with no effect on its performance, as a two-phase operation: a training phase (represented by the stochastic map $T$), where a classifier $f$ is extracted from the training set $A$, followed by a test phase (represented by $Q$), where $f$ is applied to each test instance $x_i$ and output labels $y_i$ are assigned.
The key ingredient behind its proof is that the risk is a symmetric function of the joint inputs and outputs, thus randomly permuting the inputs---and its corresponding outputs---does not affect the expected risk.
Under this randomization, the resulting protocol $\bar \P$ remains non-signalling, and thus applying the marginal protocol $\bar \P_{1|a}(y|x):=\bar \P_1(y|a,x)$ on each of the test instances will yield the same expected risk as the original protocol $\P$.
It is enlightening for our purpose to realize that all test instances are independently acted upon by maps $\bar\P_{1|a}$ that use the same sample of the training set $a$. 
As we show in the next section, this is the element of the proof that fails to hold in the quantum case.

On a fundamental note, consider the converse of Lemma~\ref{lem:inductive_classical}: any learning protocol which splits into a training phase and a test phase satisfies the non-signalling condition~\eqref{eq:nonsignalling}, so one can arguably think of the non-signalling condition as a definitory trait of inductive learning.
The advantage of this approach is many-fold.
On one hand, it allows one to focus on a much simpler set of features which fully characterize the performance of the protocol.
In addition, the training phase can be extended to provide further information relevant to assess, in advance of the test phase, the expected performance of the protocol.
This is the case in, e.g., structural risk minimization, where not only a function is chosen but also an estimator of the expected risk itself is provided~\cite{vapnik1995nature}, and confidence intervals are obtained.

\emph{Quantum learning}.
Quantum information cannot be cloned, hence the argument supporting Lemma \ref{lem:inductive_classical} breaks down.
However, it is possible to approximately clone
a quantum state,
and the quality of the clones will depend on how many copies must be produced, reaching an asymptotic limit in which each copy contains no more information than that which can be obtained by a single quantum measurement on the original system.
This idea is reflected in the seminal paper~\cite{Bae2006}, which asserts that asymptotic cloning is equivalent to state estimation succeeded by state preparation.
Intuitively, this principle hints at a plausible inductive strategy in a learning scenario where both the training set and all test instances are given as quantum states:
perform a quantum measurement $\mathcal M$ on the training set $A$, distribute the measurement outcome across all test instances $B_{1:n}$, and then use it to handle each test instance independently.
This approach has the property of being non-signalling by construction.
We will show that any symmetric non-signalling protocol can be well approximated by this strategy when the number of test instances is large.

In analogy to a classical learning problem, where an unknown probability distribution $p_{XY}$ must be mimicked by attaching appropriate labels to given random variables $X$, one may consider the most general quantum learning problem as the task of mimicking bipartite quantum states $\rho_{XY}$ by the action of a quantum channel on the marginal $\rho_X$.
The learning protocol can be thought of as a collective quantum channel $\Q$ which takes a training register $A$ and the set of test instance registers $X^{\otimes n}$ as inputs, and yields a corresponding set of output registers $Y^{\otimes n}$ (see Fig.~\ref{fig:Q} for an illustrative description of the setup).
\begin{figure}[t]
	\begin{center}
	\includegraphics[width=\columnwidth]{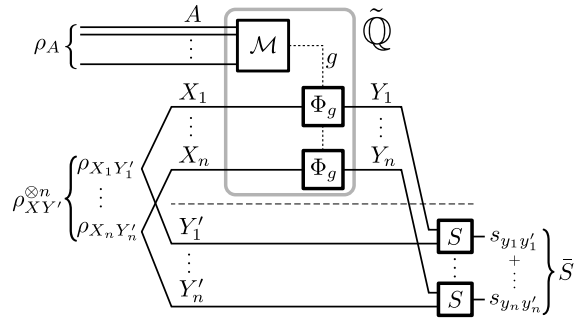}
	\end{center}
	\caption{Diagramatic representation of a generic quantum learning protocol $\Q$ (grey box), as per Definition~\ref{def:genericQ}, which be approximated by $\tilde{\Q}$, given in Eq.~\eqref{eq:Qtilde}.
Both setups take training and test instances $\rho_A$ and $\rho_{XY'}^{\otimes n}$ as inputs.
We distinguish two agents in the diagrams: the performer of the learning protocol or ``learner'', placed above the dashed horizontal line, and the ``referee'', placed below.
The learner sends the output registers $Y_{1:n}$ of the learning channel to the referee, who contrasts them with the registers $Y'_{1:n}$ and evaluates the average risk of the channel $\bar{S}$ (see Definition~\ref{def:risk_observable}).
This referee plays a role similar to the classical tester in the final phase of quantum-enhanced reinforcement learning~\cite{dunjko2016quantumenhanced}.
While the most general approach (a collective quantum channel $\Q$) in principle acts globally on all its inputs, 
its approximation $\tilde\Q$ comprises two separate phases: first (training phase) a measurement $\mathcal M$ is performed on the training set $\rho_A$, and second (test phase) the classical information $g$ obtained from the measurement is distributed among all test instances, and corresponding quantum channels $\Phi_g$ are applied locally to each one of them.}\label{fig:Q}
\end{figure}

\begin{definition}\label{def:genericQ}
A quantum learning protocol for a training set $A$ and $n$ quantum states $\rho_{XY}\in\X\otimes\Y$ is a multipartite quantum channel $\Q:\A\otimes \X_{1:n}\rightarrow \Y_{1:n}$.
A non-signalling quantum learning protocol is a quantum channel $\Q:\A\otimes \X_{1:n}\rightarrow \Y_{1:n}$ such that $\tr_{\Y_{1:i-1}\Y_{i+1:n}}[\Q(\rho_{AX_{1:n}})]$ is only a function of $\rho_{AX_{1:\hat i:n}}:=\tr_{\X_{1:i-1}\X_{i+1:n}}[\rho_{AX_{1:n}}]$, $\forall i$.
\end{definition}

This approach serves as a good starting point for generalizing several quantum learning problems, both discriminative and generative.
In particular, a quantum state classification problem may be expressed as $\rho_{XY}=\sum_{y} p_y \rho_X^{(y)}\otimes \ketbra y$, where register $X$ contains the quantum state, and $Y$ holds a classical label corresponding to the state in $X$.
This naturally encompasses the programmable quantum discriminator~\cite{Sentis2010,Sentis2012a}, but admits a much wider class of setups.
Another relevant approach is that of quantum state tomography, i.e., where a classical label $x$ is taken as a predictor for certain quantum states $\rho_Y^{(x)}$, thus $\rho_{XY}=\sum_x p_x \ketbra x\otimes\rho_Y^{(x)}$.
The task of the protocol is to learn from the training set each one of the quantum states and then produce a similar copy for each $X$ instance.
More 
generally, one could consider the task of generating genuine bipartite quantum states $\rho_{XY}$ starting from their reduced states $\rho_X$\footnote{All these tasks are meant within the framework described in Fig. \ref{fig:Q}, and therefore their performance is assessed only through registers $\Y$ and $\Y'$. In the fully quantum case of generating bipartite states $\rho_{XY}$ from $\rho_X$, this means that a risk observable would measure the dissimilarity between the states $\Q(\tr_{Y'}\rho_{XY'})$ and $\tr_X \rho_{XY'}$, instead of comparing (perhaps more naturally in other contexts) the produced bipartite state with a copy of the target bipartite state saved as reference.}.

\begin{definition}\label{def:risk_observable}
Given a \emph{risk} observable $S\in \Y\otimes\Y'$, the expected risk of protocol $\Q$ is the expectation value of the symmetrized risk observable $\bar S\in(\Y\otimes\Y')_{1:n}$ on the output of the channel $\Q$, $\E[\Q]=\tr[\Q\otimes \id_{\Y'_{1:n}}\left(\rho_A\otimes (\rho_{XY'})^{\otimes n}\right)\bar S]$.
\end{definition}

For every quantum protocol $\Q$ we can define the symmetrized protocol $\bar \Q=\frac{1}{n!}\sum_{\sigma\in S_n} {\Pi^{(\Y)}_\sigma}^\dagger \circ \Q \circ (\id_\A\otimes \Pi^{(\X)}_\sigma)$.
$\bar\Q$ is non-signalling if $\Q$ is.
In analogy with the classical case, a non-signalling quantum channel $\Q:\A\otimes\X^{\otimes n}\rightarrow\Y^{\otimes n}$ naturally admits a notion of \emph{marginalization}, thereby inducing channels for a reduced number of registers, $k\leq n$, $\Q_k:\A\otimes\X^{\otimes k}\rightarrow\Y^{\otimes k}$
(Lemma~1
in~\cite{Note1}).
Then, the expected risk can be expressed in terms of $\bar \Q_1$, $\E[\Q]=\tr[\bar \Q_1\otimes \id_{\Y'}(\rho_A\otimes\rho_{XY'})S_{YY'}]$, or the conditional channel \mbox{$\bar\Q_{1|\rho_A}:\rho_X\mapsto\bar \Q_1(\rho_A\otimes \rho_X)$}.

It is clear that the line of reasoning so far is a simple reformulation of the ideas involved in the classical arguments.
If one could implement the protocol $\bar\Q_1$ on each of the test instances then one could perform with an average performance $\E[\Q]$.
At this point, however, we encounter the fundamental roadblock that motivates this work.
The map $\rho_A\otimes \rho_{X_1}\otimes\cdots\otimes \rho_{X_n}\longmapsto \bar \Q_{1|\rho_A}(\rho_{X_1})\otimes\cdots\otimes \bar \Q_{1|\rho_A}(\rho_{X_n})$ is non-linear in $\rho_A$, so it does not reflect a physically realizable transformation.
This reflects the non-clonable nature of quantum information~\cite{Dieks1982,Wootters1982}:
it is the impossibility of cloning the training set which prevents the simultaneous application of the map $\bar \Q_{1|\rho_A}$ on the $n$ test instances.
Therefore, a generic quantum channel $\Q:\A\otimes \X_{1:n}\rightarrow\Y_{1:n}$ that distributes symetrically the system $A$ across $n$ identical parties $X$, can, at best, perform some sort of approximate cloning, which then is acted upon independently and symmetrically.
Since, asymptotically, this cloning operation becomes a measure-and-prepare process, we can reduce it to a measurement on the training set, and consider the preparation of the \emph{clones} as part of the task to be performed on each test instance.
This argument can be made formal as a de Finetti-type theorem for quantum channels~\footnote{We note that a de Finetti theorem for fully symmetric quantum channels can be found in the literature~\cite{fawzi_quantum_2015}.}.
To conclude this section, we make our statement rigorous.

\begin{theorem}[Main result]\label{thm:mainresult} Let $\Q:\A\otimes\X_{1:n}\rightarrow\Y_{1:n}$ be a non-signalling quantum channel, and let $S\in\Y\otimes\Y'$ be a local operator.
Then, there exists a POVM $\mathcal M(dg)$ on $\A$ and a set of quantum channels $\Phi_g:\X\rightarrow\Y$ such that the quantum channel $\tilde\Q$,
\begin{equation}\label{eq:Qtilde}
		\tilde{\Q}=\int \hat{\mathcal M}(dg)\otimes \Phi_g^{\otimes n},
\end{equation}
satisfies 
$\left|\E[\Q]-\E[\tilde \Q]\right|\leq \frac{\kappa}{n^{1/6}}+\left(O\frac{1}{n^{1/3}}\right),$
where the coefficient $\kappa$ depends on the dimensions of the spaces $A$, $X$ and $Y$.
\end{theorem}
The main ingredient behind this theorem is the quantum de Finetti theorem for quantum states, which can be found in~\cite{Christandl2007}.
We refer the reader to~\cite{Note1} for a detailed derivation.

Theorem~\ref{thm:mainresult} shows that, for any local operator $S$, its symmetrized expectation under the action of a non-signalling quantum channel $\Q$ can be approximated by a one-way quantum channel $\tilde \Q$ of local operations and classical communication (LOCC).
This channel amounts to performing a measurement $\mathcal M(dg)$ yielding outcome $g$ over the training set, and applying simultaneously $\Phi_g$ on each of the test instances (see Fig.~\ref{fig:Q}).
The resulting performance of both protocols, as measured by their expected risk, converge to each other as $n$ tends to infinity.

\emph{Discussion}.
The main result reported in this paper, Theorem \ref{thm:mainresult}, is a natural consequence of the symmetry implicit in the problem.
Given the fact that the performance on a multiple-instance inductive learning task is symmetric under simultaneous exchange of the test/answer pairs, a randomized permutation of the test instances will yield the same average performance.
Therefore, each protocol performs equally well as its randomized permutation protocol.
We have used this symmetry and the fact that the quantum information contained in the training set cannot be perfectly distributed over an arbitrarily large number of parties, to show that any such protocol must, effectively, be well approximated by first performing a measurement over the training set, and then distributing the outcome.

Previous works have already dwelt on this issue, that is, the contrast between coherent quantum operations and separate training and test phases for learning tasks, in various specific scenarios.
Examples are a quantum pattern matching algorithm~\cite{Sasaki2002}, quantum learning of unitary operations~\cite{Bisio2010}, and quantum learning for state classification~\cite{Sentis2012a,Sentis2014a}.
It is worth stressing that, whereas the results so far have been case-specific, we approach the problem from a very general standpoint, allowing us to discuss the broad class of inductive quantum learning protocols within a common framework.

The simplification of general quantum protocols to schemes that use LOCC has several relevant implications.
As quantum information technologies advance, coherent collective manipulation of quantum information will become accessible on a practical scale.
Nevertheless, the demonstration of a scalable, general-purpose, quantum computer is still beyond the foreseeable future.
For this reason, reducing collective approaches to simpler, local ones is of utmost importance.
With the result reported in this paper, the degree of coherence required for implementing several inductive quantum learning protocols is greatly reduced, from requiring joint coherent manipulation of both the training set and all test instances, to only the training set.

\emph{Outlook}.
Designing quantum algorithms to learn from quantum information
poses a serious challenge.
Analytical results are scarce, and numerical computations quickly become intractable.
A prominent example is quantum state discrimination, which has no known closed-form solution in general scenarios~\cite{Nakahira2015}, and only highly symmetric cases are exactly solvable~\cite{Chiribella2004}.
Reducing a generic quantum learning protocol to a single-instance one-way LOCC protocol greatly simplifies the task.
We expect that our result will allow to derive performance bounds for a variety of relevant quantum learning tasks.
Also, our quantitative approximation bounds allow for single-copy algorithms to be used as benchmarks for coherent multi-instance ones.

Another benefit of this reduction is the ability to access, without disturbance, the state of the learner 
in between the training and test phases.
This information is essential for several machine learning tasks.
For structural risk minimization~\cite{vapnik1995nature}, one uses an estimate of the expected risk, produced by evaluating the performance of a given classifier on the training set.
In the quantum setup, this approach is not directly applicable.
As the training set can only be accessed once, one can either extract information to determine the best classifier, or to assess the performance of one given classifier.
However, both tasks will generally be incompatible.
Therefore, a ``quantum black box'' -- e.g.
a fully quantum processor that takes all the inputs (training and tests) -- will, despite being the most general approach, provide only the required answers.
It is unclear how one can adapt a generic quantum black box to provide an assessment of its own performance.
Our result, nevertheless, opens the door to assessing the performance of any classifier by suitably processing the intermediate measurement outcome $g$.
We expect the result reported here will shed light on the potential and limitations of learning from quantum sources, and ultimately serve as a starting ground for developing a fully quantum theory of risk bounds in statistical learning.

A few comments on the degree of generality of our result are in order.
The convergence rate of our approximation is potentially not tight, and we expect better bounds to be achievable.
For simplicity, the approach presented here uses the operator form of Chebyshev inequality (Lemma~5 
in~\cite{Note1}), which ultimately hinders us from obtaining a better bound.
We expect a more detailed study will yield better approximations.
More importantly, our result can be extended in various ways.
A potentially very relevant practical problem is to learn quantum operations rather than \emph{states}.
This, however, can be easily addressed within the Choi matrix formalism.
A related result for learning quantum unitary operations already shows the same splitting reported here~\cite{Bisio2010}.
Indeed, the formalism of quantum combs~\cite{Chiribella2008} provides the theoretical framework for this extension, but, essentially, the most general such process will also be described by a suitable multipartite quantum system $\omega_{AB_{1:n}}$, where $A$ will now consist of input and output ports, and the maps $\Phi_g$ will be potential implementations of the learned operations.

The authors wish to thank useful discussions with Andreas Winter and Giulio Chiribella. We are grateful to the anonymous referees who provided extensive feedback that greatly improved the paper.
A.M.
is supported by the ERC (Advanced Grant IRQUAT, project number ERC-267386).
G.S.
is supported by the Spanish MINECO (project FIS2015-67161-P) and the ERC (Starting Grant 258647/GEDENTQOPT).
P.W.
acknowledges financial support from the ERC (Consolidator Grant QITBOX), MINECO (Severo Ochoa grant SEV-2015-0522 and FOQUS), Generalitat de Catalunya (SGR 875), and Fundaci\'o Privada Cellex.

\clearpage
\onecolumngrid
\appendix
\setcounter{theorem}{0}
\setcounter{lemma}{0}

\section{Supplemental material: Proof of our main result}
Our main result, Theorem~\ref{sm:thm:mainresult} (Theorem~1 in the main text), consists in showing that for every non-signalling CPTP map $\Q:\A\otimes\X_{1:n}\rightarrow\Y_{1:n}$ there is a symmetric one-way LOCC map $\Q:\A\otimes\X_{1:n}\rightarrow\bar\Y_{1:n}$ that approximately reproduces all local expectation values, and is non-signalling by construction.
The backbone of our result is the quantum de Finetti theorem, specifically in its form as it appears in \cite{Christandl2007}, which we restate here: 

\begin{theorem}[Quantum de Finetti theorem \cite{Christandl2007}]
	\label{sm:thm:definetti} Let $A$ and $B$ be quantum systems and let $\omega_{AB_{1:k}}$ be a symmetric quantum state under exchange of the $B$ systems.
	If $\omega_{AB_{1:k}}$ admits a symmetric extension $\omega_{AB_{1:n}}$ then there is a set $G$, a
	POVM $\M(dg)$ over $G$ on $A$, and a map $\phi:G\rightarrow\mathcal\B$ such that 
	\begin{align}
	\left\|\omega_{AB_{1:k}}- \int \frac{\M(dg)}{d_A}\otimes \phi_g{}^{\otimes k}\right\|_1\leq\frac{4d^2k}{n}
	\end{align}
	where $d=\dim(\B)$, $d_A=\dim(\A)$, $\phi_g\geq0,\,\forall g\in G$,
	$\M(dg)$
	only depends on the $n$-extension $\omega_{AB_{1:n}}$ and, in particular, is independent of $k$.
	$\|X\|_1=\tr|X|$ denotes the trace-norm of operator $X$.
	In general, one can take $G={\rm SU}(d^2)$.
	G and the accuracy of the approximation is independent of the dimension of $A$.
\end{theorem}

In order to apply Theorem~\ref{sm:thm:definetti} to our problem, we also use the Choi-Jamiolkowski identification between quantum states and quantum channels \cite{bengtsson2006geometry}.

\begin{theorem}[Choi] Every CP map $\Phi:\mathcal X\rightarrow \mathcal Y$ can be represented by a positive semidefinite operator $\phi\in \mathcal X\otimes \mathcal Y$, such that
	\begin{align}
	\phi=\id_\mathcal X\otimes \Phi(\Omega),
	\end{align}
	where $\ket{\Omega}=\frac{1}{\sqrt {d_{\mathcal X}}}\sum_i\ket{i,i}\in \mathcal X\otimes \mathcal X$, and $\Omega=\ketbra\Omega$.
	In addition, for any $X\in\mathcal X$ we have
	\begin{align}
	\Phi(X)=d_{\mathcal X}\, \tr_{\mathcal X}[\phi^{\top_{\mathcal X}}\, X\otimes \openone_{\mathcal Y}].
	\end{align}
	The adjoint map $\Phi^*:\Y\rightarrow\X$ is given by (we use the customary identification between $\X^*$ and $\X$ induced by the Hilbert-Schmidt product)
	\begin{align}
	\Phi^*(Y)=d_{\mathcal X}\,\tr_{\mathcal Y}[\phi^{\top_{\mathcal X}}\openone_{\mathcal X}\otimes Y].
	\end{align}
	In addition, if $\Phi$ is trace-preserving, then $\Phi^*(\openone_{\mathcal Y})=d_{\mathcal X}\,\tr_{\mathcal Y}[\phi^{\top_{\mathcal X}}]=\openone_{\mathcal X}$.
\end{theorem}

This allows us to characterize properties of channels by referring to properties of their respective Choi matrices.
The non-signalling property of a quantum channel has a direct relation with the reduced states of its Choi matrix:

\begin{lemma}\label{sm:lemma:reducedchannel}
	Let $\Q:\A\otimes \X_{1:n}\rightarrow\Y_{1:n}$ be a non-signalling quantum channel, and let $\omega_{A(XY)_{1:n}}$ be its Choi matrix.
	Then
	\begin{align}
	\tr_{Y_{k+1:n}}[\omega_{A(XY)_{1:n}}]=\omega_{A(XY)_{1:k}}\otimes \frac{\openone_{\X_{k+1:n}}}{d^{n-k}}
	\end{align}
	and $\omega_{A(XY)_{1:k}}=\tr_{(XY)_{k+1:n}}[\omega_{A(XY)_{1:n}}]$
	is the Choi matrix of the induced channel $\Q_k:\A\otimes \X_{1:k}\rightarrow\Y_{1:k}$.
\end{lemma}

Lemma~\ref{sm:lemma:reducedchannel} is proved by straightforward evaluation.

Applying 
Theorem~\ref{sm:thm:definetti} to the Choi matrix of the CPTP map $\Q$, $\omega_{A(XY)_{1:n}}$, we get an approximation to $\Q$ as described by the Choi matrix 
\begin{align}\label{sm:eq:approximation}
\eta_{A(XY)_{1:k}}=\int_G M(dg)\otimes\phi_g^{\otimes k}.
\end{align}
For $k=0$ the approximation is exact, so $\int_GM(dg)=\tr_{(XY)_{1:n}}[\omega_{A(XY)_{1:n}}]=\openone_A/d_A$, therefore $\mathcal M(dg)=d_A M(dg)$ is a POVM.
The positive semidefinite quantum states $\phi_g$
describe a family of completely positive maps $\Phi_g:\X\rightarrow\Y$.

The state $\eta_{A(XY)_{1:k}}$ does not, however, represent a quantum operation which is deterministically realizable, in the first place because $\tr_{\Y_{1:k}}[\eta_{A(XY)_{1:k}}]$ may not be $\openone_{AX_{1:k}}/d_Ad_X^k$, as is required for a trace-preserving channel.
Furthermore, a quantum channel can be implemented by 1-way LOCC iff its Choi matrix is of the form
\begin{align}
\tilde \eta_{A(XY)_{1:k}}=\int_G M(dg)\otimes\tilde \phi_g^{\otimes k},
\end{align}
where $\tr_{\Y}[\tilde \phi_g]=\openone/d$, for all $g\in G$.
This would ensure that all corresponding CP maps $\tilde \Phi_g$ are trace-preserving, and thus the channel described by $\tilde \eta_{A(XY)_{1:k}}$ can be implemented by first performing measurement $\mathcal M$ on $A$ and then applying $\tilde \Phi_g:\X\rightarrow\Y$ on each of the systems $X$.

Although one does not expect that each $\phi_g$ in Eq.~\eqref{sm:eq:approximation} satisfies
\begin{align}\label{sm:eq:tracepreserving}
\tr_{\Y}[\phi_g]\stackrel{?}=\openone_{\X}/d_X,
\end{align}
on average they approximately do.
More importantly, we now show that the outcomes $g$ are concentrated with high probability on those $\phi_g$ which almost satisfy the condition.
Let $\|\cdot\|_1$ be the trace-norm and $\|\cdot\|_\infty$ be the operator norm.

\begin{lemma}\label{sm:lemma:concentration_measure} Let $\Q$ be a non-signalling CPTP map $\Q:\A\otimes \X^{\otimes n}\rightarrow\Y^{\otimes n}$ with Choi matrix $\omega_{A(XY)_{1:n}}$, and let $M(dg)$ and $\{\phi_g\in\X\otimes\Y\}_G$ be such that
	\[
	\eta_{A(XY)_{1:k}}=\int_G M(dg)\otimes\phi_g^{\otimes k}
	\]
	is a separable approximation of $\omega_{A(XY)_{1:k}}$ such that
	\begin{align}\label{sm:eq:kapproximation}
	\left\|\eta_{A(XY)_{1:k}}-\omega_{A(XY)_{1:k}}\right\|_1\leq k \delta.
	\end{align}
	Define for all $0\leq k\leq n$ and for any subset $R\subseteq G$,
	\begin{align}
	\EV_k[R]=\int_R \tr[M(dg)] \tr_{\Y}[\phi_g]^{\otimes k}.
	\end{align}
	Then, the following holds
	\begin{enumerate}
		\item $\left\|\EV_k[G]-\openone_{\X_{1:k}}/d_X^k\right\|_1\leq k\delta.$
		\item For any $\epsilon>0$, let $R_\epsilon=\left\{g\in G|\|\tr_{\Y}[\phi_g]-\EV_1[G]\|_\infty<\epsilon\right\}$, $\bar R_\epsilon\equiv R_\epsilon\backslash G$.
		Then
		\begin{align}
		\EV_0[\bar R_\epsilon]\leq \frac{d^2_X}{\epsilon^2}\left(2\delta\left(1+{1\over d_X}\right)+\delta^2\right).
		\end{align}
		%
	\end{enumerate}
\end{lemma}

Consider the measurement $\mathcal M(dg)=d_A M(dg)$ is performed on the state $\openone_A/d_A$ yielding outcome $g$, and $\Phi_g:\X\rightarrow\Y$ is to be applied on each of the test instances.
Of course, for this to be deterministically implementable, one needs that $\Phi^*(\openone_\Y)=\openone_\X$, which amounts to $\tau_g=\tr_{\Y}[\phi_g^{\top_\X}]=\openone_\X/d_X$.
If this condition is met approximately, one can implement a suitably modified map $\tilde \Phi_g$ at the expense of actually implementing a slightly worse approximation to $\Q$.
However, if the condition is not met even approximately, the implementation cannot be expected to approximate $\Q$.
Lemma~\ref{sm:lemma:concentration_measure} shows that this case is unlikely to occur, since
\begin{align}
\Pr[\|\tau-\EV[\tau]\|_\infty\geq\epsilon]=\EV[\bar R_\epsilon] & \leq \frac{d^2_X}{\epsilon^2}\left(2\delta\left(1+{1\over d_X}\right)+\delta^2\right)\nonumber\\%
& = \frac{d^2_X}{\epsilon^2}\left(\frac{2d_A(d_X+ d_X^2)}{n}+\left(\frac{d_A d_X^2}{n}\right)^2\right).
\end{align}
Hence, one can slightly modify the operators $\phi_g$ into $\tilde\phi_g$ in order to satisfy Eq.~\eqref{sm:eq:tracepreserving} and ensure that in all cases, either $\phi_g$ and $\tilde\phi_g$ are close enough, or $g$ is unlikely enough so that the approximation still converges in $n$ to the actual channel given by $\omega_{A(XY)_{1:n}}$.
We call this a 1-way LOCC approximation.

\begin{restatable}[1-way LOCC approximation]{lemma}{tracepreserving}\label{sm:lemma:tracepreserving} Let $\Q$ be a symmetric, non-signalling CPTP map $\Q:\A\otimes \X^{\otimes n}\rightarrow\Y^{\otimes n}$ with Choi matrix $\omega_{A(XY)_{1:n}}$.
	Then there is a POVM $d_A M(dg)$ and there are states $\tilde\phi_g$ such that $\tr_{\Y}[\tilde\phi_g]=\openone_\X/d_X$ and the quantum state 
	\begin{align}
	\eta_{AXY}=\int_G M(dg)\otimes\tilde\phi_g\
	\end{align}
	is a separable approximation to $\omega_{AXY}$,
	\begin{align}
	\left\|\omega_{AXY}-\eta_{AXY}\right\|_1\leq c\, n^{-1/6} + O(n^{-1/3}).
	\end{align}
	where $c$ is a constant depending on $\dim(\X)$ and $\dim(\Y)$.
\end{restatable}
\begin{proof}[Proof of Lemma~\ref{sm:lemma:tracepreserving}]
	Let $M(dg)$ and $\{\phi_g\}$ be the factors in the de Finetti approximation to $\omega_{A(XY)_{1:k}}$, which admits a symmetric $n$-extension by assumption.
	Then they satisfy Eq.~\eqref{sm:eq:kapproximation} with $\delta=4d_X^2d_Y^2/n$.
	From Statement 1 in Lemma~\ref{sm:lemma:concentration_measure} we have
	\begin{align}
	\left\|\EV_1[G]-{\openone_\X\over d_X}\right\|_\infty\leq\left\|\EV_1[G]-{\openone_\X\over d_X}\right\|_1\leq \delta,
	\end{align}
	so that
	\begin{align}
	\EV_1[G]\geq\left({1\over d_X}-\delta\right)\openone_\X.
	\end{align}
	Therefore, for $\epsilon<{1\over d_x}-\delta$ we have
	\begin{align}
	g\in R_\epsilon\Rightarrow \tau_g\equiv\tr_\Y[\phi_g]&\geq\EV_1[G]-\epsilon\openone>0.
	\end{align}
	Thus, we can ensure that all $g\in R_\epsilon$ satisfy $\tau_g>0$.
	We can define
	\begin{align}
	\tilde \phi_g = \left\{\begin{array}{ll}
	\frac{1}{d_X}(\tau_g^{-1/2}\otimes \openone_{\Y})\, \phi_g\, (\tau_g^{-1/2}\otimes \openone_\Y )&	\mathrm{if~} g\in R_\epsilon\\
	\varphi& \mathrm{if~} g\in\bar R_\epsilon
	\end{array}\right.,
	\end{align}
	where $\varphi$ is the Choi matrix of any CPTP map $\X\rightarrow\Y$.
	By definition every $g\in R_\epsilon$ has $\tau_g\geq\EV_1[G]-\epsilon\openone$, and using $\EV_1[G]\geq\left({1\over d}-\delta\right)\openone$ we can write
	\begin{align}
	\tr[\tau_g^{1/2}]^2&\geq\left(\tr\sqrt{\EV_1[G]-\epsilon\openone_\X}\right)^2\nonumber\\
	&\geq\left(\tr[(\tfrac{1}{d_X}-\delta-\epsilon)^{1/2}\openone_\X]\right)^2\nonumber\\
	&=d_X-d_X^2(\delta+\epsilon),
	\end{align}
	Thus, Lemma~\ref{sm:lemma:trace_distance} shows that for all $g\in R_\epsilon$,
	\begin{align}
	\|\phi_g-\tilde\phi_g\|_1\leq\sqrt{d_X}\sqrt{\epsilon+\delta},
	\end{align}
	and the subadditivity of the trace distance ($\|\rho^{\otimes k}-\sigma^{\otimes k}\|_1\leq k\|\rho-\sigma\|_1$) leads to
	\begin{align}
	\|\phi_g^{\otimes k}-\tilde\phi_g^{\otimes k}\|_1\leq\left\{\begin{array}{ll}
	k\sqrt{d_X}\sqrt{\epsilon+\delta}&	\mathrm{if~} g\in R_\epsilon\\
	2& \mathrm{if~} g\in\bar R_\epsilon
	\end{array}\right.
	.
	\end{align}
	Combining this with $\|A\otimes (B-B')\|_1=\tr[A]\|B-B'\|_1$ for all $A\geq0$,
	\begin{align}
	\|M(dg)\otimes (\phi_g^{\otimes k}-\tilde\phi_g^{\otimes k})\|_1&=\tr[M(dg)]\|\phi_g^{\otimes k}-\tilde\phi_g^{\otimes k}\|_1 \,,
	\end{align}
	and the triangle inequality we get
	\begin{align}
	\left\|\int_GM(dg)\otimes(\phi_g^{\otimes k}-\tilde\phi_g^{\otimes k})\right\|_1&\leq \int_G\tr[M(dg)]\|\phi_g^{\otimes k}-\tilde\phi_g^{\otimes k}\|_1\nonumber\\
	&\leq k\sqrt{d_X}\sqrt{\epsilon+\delta} \int_{R_\epsilon} \tr[M(dg)]  + 2 \int_{\bar R_\epsilon}\tr[M(dg)]\nonumber\\
	&\leq k\sqrt{d_X}\sqrt{\epsilon+\delta}+2 \EV_0[\bar R_\epsilon]\nonumber\\
	&\leq k\sqrt{d_X}\sqrt{\epsilon+\delta}+2\frac{d_X^2}{\epsilon^2}\left(2\delta\left(1+{1\over d_X}\right)+\delta^2\right).
	\end{align}
	
	Taking $k=1$ and using the triangle inequality we get
	\begin{align}
	\left\|\omega_{AXY}-\eta_{AXY}\right\|_1&\leq\left\|\omega_{AXY}-\int_G M(dg)\otimes  \phi_g\right\|_1+\left\|\int_GM(dg)\otimes(\phi_g-\tilde\phi_g)\right\|_1\nonumber\\
	&\leq \delta + \sqrt{d_X}\sqrt{\epsilon+\delta}+2\frac{d_X^2}{\epsilon^2}\left(2\delta\left(1+{1\over d_X}\right)+\delta^2\right).
	\end{align}
	Chosing $\epsilon=\delta^{1/3}$ and expanding around $\delta=0$ up to leading order we get
	\begin{align}
	\left\|\omega_{AXY}-\int_G M(dg)\otimes \tilde \phi_g\right\|_1\leq \sqrt{d_X}\delta^{1/6} + O(\delta^{1/3}),
	\end{align}
	which using $\delta=4(d_Xd_Y)^2/n$ leads to
	\begin{align}
	\left\|\omega_{AXY}-\int_G M(dg)\otimes \tilde \phi_g\right\|_1\leq 4^{1/6}d_X^{5/6}d_Y^{1/3} \frac{1}{n^{1/6}}+O(n^{-1/3})
	\end{align}
	the desired result.
\end{proof}

Having established a 1-way LOCC approximation bound for any symmetric non-signalling channel, we can now proceed to prove our main result (Theorem 1 in the main text):

\begin{theorem}[Main result]\label{sm:thm:mainresult} Let $\Q:\A\otimes\X_{1:n}\rightarrow\Y_{1:n}$ be a non-signalling quantum channel, and let $S\in\Y\otimes\Y'$ be a local operator.
	Then, there exists a POVM $\mathcal M(dg)$ on $\A$ and a set of quantum channels $\Phi_g:\X\rightarrow\Y$ such that the quantum channel $\tilde\Q$,
	\begin{equation}\label{sm:eq:Qtilde}
	\tilde{\Q}=\int \hat{\mathcal M}(dg)\otimes \Phi_g^{\otimes n},
	\end{equation}
	satisfies $\left|\E[\Q]-\E[\tilde \Q]\right|\leq O\left(\frac{1}{n^{1/6}}\right).
	$
\end{theorem}

\begin{proof}[Proof of Theorem~\ref{sm:thm:mainresult}]
	We want to obtain approximation bounds for 
	\begin{align}
	\E[\Q]\equiv\tr[\bar S\, (\Q\otimes \id_{\Y'_{1:n}})(\rho_{A(XY')_{1:n}})].
	\end{align}
	The specific form of $\rho_{A(XY')_{1:n}}$ is irrelevant for our purposes, besides symmetry among the $XY'$ parties.
	Expressing $\E[\Q]$ in terms of the symmetrized local channel $\bar\Q_1$, and in turn, in terms of its Choi matrix, we have
	\begin{align}\label{sm:eq:EQ}
	\E[\Q]&=\tr[S\,\bar\Q_1\otimes\id_{\Y'}(\rho_{AXY'})]\nonumber\\
	&=d_Ad_X\tr_{\Y\Y'}[S\,\tr_{\A\X}[(\omega_{AXY}\otimes \openone_{\Y'})^{\top_{\A\X}}(\rho_{AXY'}\otimes\openone_{\Y})]]\nonumber\\
	&=d_Ad_X\tr_{\Y\Y'}[S\,\tr_{\A\X}[(\omega_{AXY}\otimes \openone_{\Y'})(\rho_{AXY'}\otimes\openone_{\Y})^{\top_{\A\X}}]]\nonumber\\
	&=d_Ad_X\tr_{\Y\Y'}[\tr_{\A\X}[\openone_{\A\X}\otimes S\,(\omega_{AXY}\otimes \openone_{\Y'})(\rho_{AXY'}\otimes\openone_{\Y})^{\top_{\A\X}}]]\nonumber\\
	&=d_Ad_X\tr_{\A\X\Y\Y'}[\openone_{\A\X}\otimes S\,(\omega_{AXY}\otimes \openone_{\Y'})(\rho_{AXY'}\otimes\openone_{\Y})^{\top_{\A\X}}]]\nonumber\\
	&=d_Ad_X\tr_{\A\X\Y}[\omega_{AXY}\,\tr_{\Y'}[(\rho_{AXY'}\otimes\openone_{\Y})^{\top_{\A\X}}(\openone_{\A\X}\otimes S)]].
	\end{align}
	
	To ease the notation, it is convenient to define $R=\tr_{\Y'}[(\rho_{AXY'}\otimes\openone_{\Y})^{\top_{\A\X}}(\openone_{\A\X}\otimes S)]\in\A\otimes\X\otimes \Y$, so that Eq.~\eqref{sm:eq:EQ} reads
	\begin{align}
	\E[\Q]=d_Ad_X\tr[\omega_{AXY} R].
	\end{align}
	Using Lemma~\ref{sm:lemma:tracepreserving} we can replace $\omega_{A(XY)_{1:k}}$ by its 1-way LOCC approximation $\eta_{AXY}$,
	\begin{align}
	\E[\tilde \Q]=d_Ad_X\tr[\eta_{AXY} R],
	\end{align}
	which satisfies
	\begin{align}
	\left\|\E[\Q]-\E[\tilde\Q]\right\|&\leq d_A d_X \big| \tr[(\omega_{AXY}-\eta_{AXY}) R]\big|\\
	&\leq d_Ad_X\left\|\omega_{A(XY)}-\int_G M(dg)\otimes\tilde \phi_g\right\|_1\|R\|_\infty\nonumber\\
	&\leq \left(4^{1/6} d_A  d_X^{11/6}d_Y^{1/3} \frac{1}{n^{1/6}}+O(1/n^{1/3})\right)\|R\|_\infty.
	\end{align}
	Finally, we can absorb the constant $\|R\|_\infty$ into the factors preceeding $1/n^{1/6}$.
\end{proof}

\section{Proofs of Lemmas and Theorem~\ref{sm:thm:definetti}}

We restate and prove Lemma~1 in the main text. We also mention that a related but more general result on a de Finetti theorem for non-signalling classical conditional probability distributions can be found in~\cite{Christandl2009a}.

\begin{lemma}\label{sm:lem:inductive_classical}For every inductive learning protocol $\P$ that assigns labels $y_i$ to test instantes $x_i$, there exists a set $F$ of classifying functions $f:X\rightarrow Y$ and stochastic maps $T(f|A)$, $Q(y|x,f)=\delta_{y,f(x)}$ such that the inductive protocol $\tilde \P$
	\begin{align}
	\tilde \P(y_{1:n}|A,x_{1:n})=\sum_f  \left[\prod_{i=1}^n Q(y_i|x_i, f)\right] T(f|A)\nonumber
	\end{align}
	has expected risk $\E[\P|A]= \E[{\tilde \P}|A]$ for all $A$.
\end{lemma}

\begin{proof}[Proof of Lemma~\ref{sm:lem:inductive_classical}] Consider the expected risk $\E[\P]$ of protocol $\P$.
	Let $\sigma\in S_n$ be any permutation of $n$ elements, and let the $\P^{(\sigma)}$ be the accordingly permuted protocol
	\begin{align}
	\P^{(\sigma)}(y_{1:n}|A,x_{1:n})=\P(y_{\sigma(1):\sigma(n)}|A,x_{\sigma(1):\sigma(n)}).
	\end{align}
	Furthermore, let $\bar \P$ be the symmetrized protocol,
	\begin{align}
	\bar \P(y_{1:n}|A,x_{1:n})=\frac{1}{n!}\sum_{\sigma\in S_n} \P^{(\sigma)}(y_{1:n}|A,x_{1:n}).
	\end{align}
	It follows trivially that 
	\begin{align}
	\E[\P|A]=\E[\P^{(\sigma)}|A]=\E[\bar \P|A],\quad \forall\sigma\in S_n, A.
	\end{align}
	%
	One can define the marginal maps $\bar \P_i$, which are all equal, so we refer to them as $\bar \P_1$,
	\begin{align}
	\bar \P_1(y|A,x_{1:n})=\sum_{y_{2,n}} \bar \P(y,y_{2:n}|A,x_{1:n})\,.
	\end{align}
	Since $\P$ is non-signalling, so is $\bar\P$, namely $\bar\P_1$ satisfies the condition 
	\begin{align}
	\bar\P_1(y|A,x_{1:n})=\bar\P_1(y_i|A,x,x'_{2:n}) \,,\;\forall x'_{2:n} \,
	\end{align}
	and so we can simply write $\bar\P_1(y|A,x)$.
	The conditional expected risk can be expressed in terms of $\bar \P_1$,
	\begin{align}
	\E[\P|A]=\sum_{x,y,y'}\bar \delta_{y, y'} \bar \P_1(y|A,x)P_{XY}(x,y').
	\end{align}
	Considering $A$ fixed, $\bar \P_1(y|A,x)$ is a stochastic map from $X$ to $Y$, and thus it is a convex combination of deterministic maps $Q_f(y|x)=\delta_{y,f(x)}$ for some set of functions $f\in\mathcal F$, i.e.
	\begin{align}
	\bar \P_1(y|A,x)=\sum_f \mu_A(f) \,Q_f(y|x),
	\end{align}
	where $\mu_A$ is a probability measure that depends on $A$.
	Then
	\begin{align}
	\E[\P|A]=\sum_f \mu_A(f)\E[Q_f|A].
	\end{align}
	Thus, the stochastic maps $Q(y|x,f)=Q_f(y|x)$ and $T(f|A)=\mu_A(f)$ can be combined into the protocol
	\begin{align}
	\tilde\P(y_{1:n}|A,x_{1:n})=\sum_f  \left[\prod_{i=1}^n Q(y_i|x_i, f)\right] T(f|A),
	\end{align}
	which achieves
	\begin{align}
	\E[\tilde\P|A]&=\E\left[\sum_f  \prod_{i=1}^n Q(y_i|x_i, f)T(f|A)\middle|A\right]\nonumber\\
	&=\	\sum_{x_{1:n},y_{1:n},y'_{1:n}}
	\frac{1}{n}\left(s_{y_1,y'_1}+\cdots+s_{y_n,y'_n}\right)\sum_f  \prod_{i=1}^n P_{XY}(x_1,y'_1)\cdots P_{XY}(x_n,y'_n)\nonumber\\
	&=\sum_{x_1,y_1,y'_1}\bar\delta_{y_1y'_1}\sum_f Q(y_1|x_1,f)T(f|A)P_{XY}(x_1,y'_1)\nonumber\\
	&=\sum_{x_1,y_1,y'_1}\bar\delta_{y_1y'_1}\bar \P_1(y_1|A,X) P_{XY}(x_1,y'_1)\nonumber\\
	&=\E[\bar \P|A].
	\end{align}
\end{proof}


The following proof of Theorem~\ref{sm:thm:definetti} reproduces that of the original paper~\cite{Christandl2007}, where, as suggested, a probability measure is replaced by an operator-valued measure.

\begin{proof}[Proof of Theorem~\ref{sm:thm:definetti}]
	Let us start by assuming $\omega_{AB_{1:k}}$ admits a pure state extension $\Psi_{AB_{1:n}}=\ketbra{\Psi_{AB_{1:n}}}$.
	Then
	\begin{align}
	\ket{\Psi_{AB_{1:n}}}\in\mathcal H_\A\otimes \mathcal H_{\sym(n)},
	\end{align}
	where $\mathcal H_{\sym(n)}$ is the symmetric subspace of $\mathcal H_{B_{1:n}}$.
	Let also $d\equiv\dim(\mathcal H_B)$.
	
	Let $g$ be a generic ${\rm SU}(d)$ element, $\ket 0$ a reference state in $\mathcal H_B$, and $dg$ the Haar measure on ${\rm SU}(d)$.
	Let $\ket{\phi_g}=U_g \ket 0$ and use $\phi_g=\ketbra{\phi_g}$.
	
	For any $1\leq k\leq n$, let $E_k^g=\dim \mathcal H_{\sym(k)} \phi_g^{\otimes k}$ be a POVM in $ \mathcal H_{\sym(k)}$, such that 
	\begin{align}
	\int dg \,E_k^g=\openone_{\mathcal H_{\sym(k)}}.
	\end{align}
	
	This allows to write
	\begin{align}
	\omega_{AB_{1:k}}=\int w_g \omega^g_{AB_{1:k}} dg,
	\end{align}
	where $w_g \omega^g_{AB_{1:k}}$ is the residual state on $AB_{1:k}$ when measuring $E_{n-k}$
	\begin{align}
	w_g \omega^g_{AB_{1:k}}=\tr_{\B_{k+1:n}}[\openone_\A\otimes\openone_{\B_{1:k}}\otimes E_{n-k}^g \Psi_{AB_{1:n}}].
	\end{align}
	
	Then $\omega_{AB_{1:k}}$  is close to a convex combination of separable and $B$-iid states $\int M(dg)\otimes \phi_g^{\otimes k}$, with a distribution $M(dg)$ independent of $k$, namely
	\begin{align}
	\Delta_k=\omega_{AB_{1:k}}-\int M(dg) \otimes \phi_g^{\otimes k}
	\end{align}
	is close to zero in trace-norm.
	The operator-valued measure $M(dg)$ is given by
	\begin{align}
	M(dg)=\tr_{\B_{1:n}}[\openone_\A\otimes E_n^g \,\Psi_{AB_{1:n}}]dg.
	\end{align}
	
	We now bound $\|\Delta_k\|_1=\|S-\delta\|_1$, where
	\begin{align}
	S&=\int w_g \omega^g_{AB_{1:k}} dg-\frac{\dim \mathcal H_{\sym(n-k)}}{\dim \mathcal H_{\sym(n)}}\int M(dg)\otimes \phi_g^{\otimes k},\\
	\delta&=\left(1-\frac{\dim \mathcal H_{\sym(n-k)}}{\dim \mathcal H_{\sym(n)}}\right)\int M(dg)\otimes \phi_g^{\otimes k}.
	\end{align}
	One can readily check that 
	\begin{align}
	\|\delta\|_1&=\left(1-\frac{\dim \mathcal H_{\sym(n-k)}}{\dim \mathcal H_{\sym(n)}}\right)\left\|\int M(dg)\otimes \phi_g^{\otimes k}\right\|_1\nonumber\\
	&=\left(1-\frac{\dim \mathcal H_{\sym(n-k)}}{\dim \mathcal H_{\sym(n)}}\right)\int \tr[M(dg)\otimes \phi_g^{\otimes k}]\nonumber\\
	&=\left(1-\frac{\dim \mathcal H_{\sym(n-k)}}{\dim \mathcal H_{\sym(n)}}\right)\int \tr[M(dg)]\nonumber\\
	&=\left(1-\frac{\dim \mathcal H_{\sym(n-k)}}{\dim \mathcal H_{\sym(n)}}\right)\int \tr[\openone_\A\otimes E_n^g \,\Psi_{AB_{1:n}}]dg\nonumber\\
	&=\left(1-\frac{\dim \mathcal H_{\sym(n-k)}}{\dim \mathcal H_{\sym(n)}}\right)\tr[\Psi_{AB_{1:n}}]\nonumber\\
	&=1-\frac{\dim \mathcal H_{\sym(n-k)}}{\dim \mathcal H_{\sym(n)}}.
	\end{align}
	On the other hand,
	\begin{align}
	\frac{\dim \mathcal H_{\sym(n-k)}}{\dim \mathcal H_{\sym(n)}}&\tr_{\B_{1:n}}[\openone_\A\otimes E_n^g \,\Psi_{AB_{1:n}}]=\dim \mathcal H_{\sym(n-k)}\tr_{\B_{1:n}}[\openone_\A\otimes \phi_g^{\otimes n} \,\Psi_{AB_{1:n}}]\nonumber\\
	&=\dim \mathcal H_{\sym(n-k)}\tr_{\B_{1:n}}[\openone_\A\otimes \phi_g^{\otimes k}\otimes\phi_g^{\otimes (n-k)}\,\Psi_{AB_{1:n}}]\nonumber\\
	&=\bra{\phi_g^{\otimes k}}\tr_{\B_{k+1:n}}[\openone_\A\otimes \openone_{\B_{1:k}}\otimes E_{(n-k)}^g\,\Psi_{AB_{1:n}}]\ket{\phi_g^{\otimes k}}\nonumber\\
	&=w_g \bra{\phi_g^{\otimes k}}\omega^g_{AB_{1:k}}\ket{\phi_g^{\otimes k}}.
	\end{align}
	Notice that this is an operator in $\A$.
	With this we have
	\begin{align}
	S&=\int w_g\omega^g_{AB_{1:k}}dg -\frac{\dim \mathcal H_{\sym(n-k)}}{\dim \mathcal H_{\sym(n)}}\int \tr_{\B_{1:n}}[\openone_\A\otimes E_n^g \,\Psi_{AB_{1:n}}]\otimes \phi_g^{\otimes k}dg\nonumber\\
	&=\int w_g\omega^g_{AB_{1:k}}dg -\int \bra{\phi_g^{\otimes k}}w_g \omega^g_{AB_{1:k}}\ket{\phi_g^{\otimes k}}\otimes \phi_g^{\otimes k}dg\nonumber\\
	&=\int w_g\left(\omega^g_{AB_{1:k}} - \left[\openone_\A\otimes\phi_g^{\otimes k}\right] \omega^g_{AB_{1:k}}\left[\openone_\A\otimes\phi_g^{\otimes k}\right]\right)dg.
	\end{align}
	We now use $A-BAB=(A-BA)+(A-AB)-(1-B)A(1-B)$, so that we are interested in expressions of the form
	\begin{align}
	\alpha&=\int w_g\left[\openone_{\A\B_{1:k}}-\openone_\A\otimes\phi_g^{\otimes k}\right]\omega^g_{AB_{1:k}}dg,\\
	\gamma&=\int w_g\left[\openone_{\A\B_{1:k}}-\openone_\A\otimes\phi_g^{\otimes k}\right]\omega^g_{AB_{1:k}}\left[\openone_{\A\B_{1:k}}-\openone_\A\otimes\phi_g^{\otimes k}\right]dg,
	\end{align}
	so that $S=\alpha+\alpha^\dagger+\gamma$.
	Using 
	\begin{align}
	w_g\left[\openone_\A\otimes\phi_g^{\otimes k}\right] \omega^g_{AB_{1:k}}&=\tr_{\B_{k+1:n}}[\openone_\A\otimes\phi_g^{\otimes k}\otimes E_{n-k}^g\Psi_{AB_{1:n}}]\\
	&=\frac{\dim \mathcal H_{\sym(n-k)}}{\dim \mathcal H_{\sym(n)}}\tr_{\B_{k+1:n}}[\openone_\A\otimes E_{n}^g\,\Psi_{AB_{1:n}}]
	\end{align}
	we have
	\begin{align}
	\alpha=\left(1-\frac{\dim \mathcal H_{\sym(n-k)}}{\dim \mathcal H_{\sym(n)}}\right)\tr_{\B_{k+1:n}}[\Psi_{AB_{1:n}}],
	\end{align}
	thus
	\begin{align}
	\|\alpha\|_1=\|\alpha^\dagger\|_1=1-\frac{\dim \mathcal H_{\sym(n-k)}}{\dim \mathcal H_{\sym(n)}}.
	\end{align}
	
	For $\gamma$ we use the relation $\tr[PX]=\|PXP\|_1$.
	This together with convexity of the trace yields
	\begin{align}
	\|\gamma\|_1&\leq\int\left\|\left[\openone-\openone_\A\otimes\phi_g^{\otimes k}\right]\omega^g_{AB_{1:k}}\left[\openone-\openone_\A\otimes\phi_g^{\otimes k}\right]\right\|_1dg\nonumber\\
	&=\tr\left[\int w_g\left[\openone-\openone_\A\otimes\phi_g^{\otimes k}\right]\omega^g_{AB_{1:k}}\right]\nonumber\\
	&=\|\alpha\|_1\,.
	\end{align}
	
	In summary, we can write
	\begin{align}
	\|\Delta_k\|_1\leq\|S\|_1+\|\delta\|_1\leq\|\alpha\|_1+\|\alpha^\dagger\|_1+\|\gamma\|_1+\|\delta\|_1,
	\end{align}
	where each of the norms is upper-bounded by
	\begin{align}
	\frac{1}{2}\left(1-\frac{\dim \mathcal H_{\sym(n-k)}}{\dim \mathcal H_{\sym(n)}}\right).
	\end{align}
	Using
	\begin{align}
	\frac{\dim \mathcal H_{\sym(n-k)}}{\dim \mathcal H_{\sym(n)}}\geq1-\frac{dk}{n},
	\end{align}
	we finally get
	\begin{align}
	\|\Delta_k\|_1\leq\frac{4dk}{n}.
	\end{align}
	This proves the statement in case the extension $\omega_{AB_{1:n}}=\ketbra{\Psi_{AB_{1:n}}}$ is a pure state.
	In case no such pure state exists, let $\omega_{AB_{1:n}}$ be a symmetric purification, generally mixed, and let 
	\begin{align}
	\ket{\Psi_{AB_{1:n}A'B'_{1:n}}}=\sqrt{\omega_{AB_{1:n}}}\otimes\openone_{\A'\B'_{1:n}}\ket{\Phi}
	\end{align}
	be its purification, with $\ket{\Phi}\in\left(\mathcal H_A\otimes\mathcal H_B^{\otimes n}\right)\otimes \left(\mathcal H_{A'}\otimes \mathcal H_{B'}^{\otimes n}\right)$ being the maximally entangled state among $AB_{1:n}$ and $A'B'_{1:n}$.
	The state $\Psi_{AB_{1:n}A'B'_{1:n}}$ is symmetric under exchange of $BB'$ systems, and has reductions
	\begin{align}
	\omega_{AB_{1:k}A'B'_{1:k}}=\tr_{(\B\B')_{k+1:n}}[\Psi_{AB_{1:n}A'B'_{1:n}}],
	\end{align}
	hence there is a measure $\bar M(dg)$ over ${\rm SU}(d^2)$ in $\A\A'$ such that
	\begin{align}
	\left\|\omega_{AB_{1:k}A'B'_{1:k}}-\int \bar M(dg)\otimes \bar\phi_g^{\otimes k}\right\|_1\leq \frac{4d^2k}{n},
	\end{align}
	where $\ket{\bar\phi_g}=U_g\ket{\bar 0}\in\mathcal H_B\otimes \mathcal H_{B'}$, $\ket{\bar 0}$ is a  reference pure state in $\mathcal H_B\otimes \mathcal H_{B'}$, and $U_g$ is a generic ${\rm SU}(d^2)$ element.
	Since the partial trace does not increase the trace-norm,  the statement of the theorem is recovered by tracing out the primed systems.\end{proof}

\begin{lemma}[Operator Chebyshev inequality]\label{sm:lemma:chebyshev}
	Let $X\in\B$ be an self-adjoint operator-valued random variable with expectation $\EV[X]=\mu$.
	Then,
	\begin{align}
	\Pr[\|X-\mu\|_\infty\geq \epsilon]\leq\frac{d_\B^2}{\epsilon^2} \|\EV[X\otimes X]-\mu\otimes\mu\|_\infty.
	\end{align}
\end{lemma}
\begin{proof}[Proof of Lemma~\ref{sm:lemma:chebyshev}]
	Let us first establish an intermediate result,
	\begin{align}
	\EV[\|X-\mu\|_\infty^2]&=\int \mu(dX)\|X-\mu\|_\infty^2\nonumber\\
	&\geq\int_{\|X-\mu\|_\infty\geq \epsilon}  \mu(dX)\|X-\mu\|_\infty^2\nonumber\\
	&\geq \epsilon^2\int_{\|X-\mu\|_\infty\geq \epsilon}\mu(dX)\nonumber\\
	&=\epsilon^2 \Pr[\|X-\mu\|_\infty\geq \epsilon]\,,\\
	\Pr[\|X-\mu\|_\infty\geq \epsilon]&\leq\frac{\EV[\|X-\mu\|_\infty^2]}{\epsilon^2}.
	\end{align}
	Furthermore, we can upper bound $\EV[\|X-\mu\|_\infty^2]$ in terms of $\EV[X\otimes X]$ and $\mu$,
	\begin{align}
	\|X-\mu\|_\infty^2&=\|(X-\mu)(X-\mu)\|_\infty\nonumber\\
	&\leq\|(X-\mu)(X-\mu)\|_1\nonumber\\
	&=\tr[(X-\mu)(X-\mu)]\nonumber\\
	&=\bra{\Phi}(X-\mu)\otimes(X-\mu)^\top\ket{\Phi}\nonumber\\
	&=\tr[(X-\mu)\otimes(X-\mu) \,\Phi^{\top_2}]\nonumber\\
	&=\tr[\left(X\otimes X-\mu\otimes\mu-(X-\mu)\otimes \mu-\mu\otimes (X-\mu)\right)\,\Phi^{\top_2}].
	\end{align} 
	Taking the expectation we can commute it inside the trace,
	\begin{align}
	\EV[\|X-\mu\|_\infty^2]&\leq\tr[\EV[X\otimes X-\mu\otimes\mu-(X-\mu)\otimes \mu-\mu\otimes (X-\mu)]\,\Phi^{\top_2}]\nonumber\\
	&=\tr[(\EV[X\otimes X]-\mu\otimes\mu)\,\Phi^{\top_2}]\nonumber\\
	&\leq \|\EV[X\otimes X]-\mu\otimes\mu\|_\infty\,\|\Phi^{\top_2}\|_1.
	\end{align}
	
	The swap operator $\Phi^{\top_2}$ has eigenvalues $\pm1$, so its trace-norm is $d_\B^2$.
	This concludes the proof.

\end{proof}

\begin{proof}[Proof of Lemma~\ref{sm:lemma:concentration_measure}]
	Statement 1 is proven by straightforward evaluation and exploiting the contractivity of the partial trace,
	\begin{align}
	\left\|\EV_k[G]-\frac{\openone_{\X_{1:k}}}{d_X^k}\right\|_1&=\left\|\tr_{\A\Y_{1:k}}\left[\int_GM(dg)\otimes\phi_g^{\otimes k}-\omega_{A(XY)_{1:k}}\right]\right\|_1\nonumber\\
	&=\left\|\tr_{\A\Y_{1:k}}\left[\eta_{A(XY)_{1:k}}-\omega_{A(XY)_{1:k}}\right]\right\|_1\nonumber\\
	&\leq\left\|\eta_{A(XY)_{1:k}}-\omega_{A(XY)_{1:k}}\right\|_1\nonumber\\
	&\leq k\delta.
	\end{align}

	For Statement 2, notice that $\EV_0$ is a probability measure on $G$.
	It is convenient to define the operator-valued random variable $\tau_g=\tr_{\Y}[\phi_g^{\top_\X}]$, with $\EV_1[G]$ being its mean.
	We now apply the Chebyshev-type inequality  (Lemma \ref{sm:lemma:chebyshev}) to obtain
	\begin{align}
	\EV_0[\bar R_\epsilon]\leq \frac{d_X^2}{\epsilon^2}\|\EV_2[G]-\EV_1[G]^{\otimes 2}\|_\infty.
	\end{align}
	
	Computing the bound,
	\begin{align}
	\EV_2[G]-\EV_1[G]^{\otimes 2}&=\EV_2[G]-{\openone_\X\otimes\openone_\X\over d_X^2}-\left(\EV_1[G]-{\openone_\X\over d_X}\right)^{\otimes 2}\nonumber\\
	&\quad-\left(\EV_1[G]-{\openone_\X\over d_X}\right)\otimes {\openone_\X\over d_X}-{\openone_\X\over d_X}\otimes\left(\EV_1[G]-{\openone_\X\over d_X}\right).
	\end{align}
	Thus,
	\begin{align}
	\|\EV_2[G]-\EV_1[G]^{\otimes 2}\|_\infty&\leq \left\|\EV_2[G]-{\openone_2\over d_X^2}\right\|_\infty+\left\|\left(\EV_1[G]-{\openone\over d_X}\right)^{\otimes 2}\right\|_\infty\nonumber\\
	&~~+2\left\|\left(\EV_1[G]-{\openone\over d_X}\right)\otimes {\openone\over d_X}\right\|_\infty\nonumber\\
	&\leq 2\left(1+{1\over d_X}\right)\delta+\delta^2,
	\end{align}
	where the first step just applies the triangle inequality, and the second uses the bound $\|X\|_\infty\leq\|X\|_1$.
	Then,
	\begin{align}
	\EV_0[\bar R_\epsilon]\leq \frac{d_X^2}{\epsilon^2}\left[2\left(1+{1\over d_X}\right)\delta+\delta^2\right].
	\end{align}
	On the other hand we have
	\begin{align}
	\EV_0[G]=\int_G\tr[M(dg)]=\tr[\omega_{A}]=1,
	\end{align}
	so that $\EV_0[R_\epsilon]=\EV_0[G]-\EV_0[\bar R_\epsilon]$ proves Statement 2.
\end{proof}

\begin{lemma}\label{sm:lemma:trace_distance} Let $\phi\in \X\otimes \Y$ be a quantum state with $\tau=\tr_{\Y}[\phi]>0$.
	Then the state 
	\begin{align}
	\tilde\phi=\frac{1}{d_X}(\tau^{-1/2}\otimes \openone_{\Y})\, \phi\, (\tau^{-1/2}\otimes \openone_{\Y})
	\end{align}
	satisfies 
	\begin{align}
	\tr_{\Y}[\tilde \phi]&=\frac{\openone_\X}{d_X},\\
	\|\phi-\tilde\phi\|_1&\leq\sqrt{1-\frac{1}{d_X}\tr[\tau^{1/2}]^2}.
	\end{align}
\end{lemma}
\begin{proof}[Proof of Lemma~\ref{sm:lemma:trace_distance}]
	Using $\|\rho-\sigma\|_1\leq\sqrt{1-F(\rho,\sigma)}$, where $F(\rho,\sigma)=(\tr\sqrt{\rho^{1/2}\sigma\rho^{1/2}})^2$ and computing the fidelity we have
	\begin{align}
	F(\phi,\tilde\phi)&=\left(\tr\sqrt{\phi^{1/2}\, \tilde \phi\, \phi^{1/2}}\right)^2\nonumber\\
	&=\frac{1}{d_X}\left(\tr\sqrt{\phi^{1/2}\,  (\tau^{-1/2}\otimes \openone_{\Y})\, \phi\, (\tau^{-1/2}\otimes \openone_{\Y} )\, \phi^{1/2}}\right)^2\nonumber\\
	&=\frac{1}{d_X}\left(\tr\sqrt{\phi^{1/2}\,  (\tau^{-1/2}\otimes \openone_{\Y})\, \phi^{1/2}\,\phi^{1/2}\, (\tau^{-1/2}\otimes \openone_{\Y} )\, \phi^{1/2}}\right)^2\nonumber\\
	&=\frac{1}{d_X}\left(\tr\sqrt{\left(\phi^{1/2}\,  (\tau^{-1/2}\otimes \openone_{\Y})\, \phi^{1/2}\right)^2}\right)^2\nonumber\\
	&=\frac{1}{d_X}\tr[\phi^{1/2}\,  (\tau^{-1/2}\otimes \openone_{\Y})\, \phi^{1/2}]^2\nonumber\\	
	&=\frac{1}{d_X}\tr[\phi\, (\tau^{-1/2}\otimes \openone_{\Y})]^2\nonumber\\	
	&=\frac{1}{d_X}\tr_\X[\tau^{-1/2}\,\tr_{\Y}[\phi]]^2\nonumber\\
	&=\frac{1}{d_X}\tr[\tau^{-1/2}\,\tau]^2\nonumber\\
	&=\frac{1}{d_X}\tr[\tau^{1/2}]^2.
	\end{align}
	This concludes the proof.
\end{proof}


\begin{thebibliography}{37}%
	\makeatletter
	\providecommand \@ifxundefined [1]{%
		\@ifx{#1\undefined}
	}%
	\providecommand \@ifnum [1]{%
		\ifnum #1\expandafter \@firstoftwo
		\else \expandafter \@secondoftwo
		\fi
	}%
	\providecommand \@ifx [1]{%
		\ifx #1\expandafter \@firstoftwo
		\else \expandafter \@secondoftwo
		\fi
	}%
	\providecommand \natexlab [1]{#1}%
	\providecommand \enquote  [1]{``#1''}%
	\providecommand \bibnamefont  [1]{#1}%
	\providecommand \bibfnamefont [1]{#1}%
	\providecommand \citenamefont [1]{#1}%
	\providecommand \href@noop [0]{\@secondoftwo}%
	\providecommand \href [0]{\begingroup \@sanitize@url \@href}%
	\providecommand \@href[1]{\@@startlink{#1}\@@href}%
	\providecommand \@@href[1]{\endgroup#1\@@endlink}%
	\providecommand \@sanitize@url [0]{\catcode `\\12\catcode `\$12\catcode
		`\&12\catcode `\#12\catcode `\^12\catcode `\_12\catcode `\%12\relax}%
	\providecommand \@@startlink[1]{}%
	\providecommand \@@endlink[0]{}%
	\providecommand \url  [0]{\begingroup\@sanitize@url \@url }%
	\providecommand \@url [1]{\endgroup\@href {#1}{\urlprefix }}%
	\providecommand \urlprefix  [0]{URL }%
	\providecommand \Eprint [0]{\href }%
	\providecommand \doibase [0]{http://dx.doi.org/}%
	\providecommand \selectlanguage [0]{\@gobble}%
	\providecommand \bibinfo  [0]{\@secondoftwo}%
	\providecommand \bibfield  [0]{\@secondoftwo}%
	\providecommand \translation [1]{[#1]}%
	\providecommand \BibitemOpen [0]{}%
	\providecommand \bibitemStop [0]{}%
	\providecommand \bibitemNoStop [0]{.\EOS\space}%
	\providecommand \EOS [0]{\spacefactor3000\relax}%
	\providecommand \BibitemShut  [1]{\csname bibitem#1\endcsname}%
	\let\auto@bib@innerbib\@empty
	\bibitem [{\citenamefont {Wittek}(2014)}]{wittek2014qml}%
	\BibitemOpen
	\bibfield  {author} {\bibinfo {author} {\bibfnamefont {P.}~\bibnamefont
			{Wittek}},\ }\href@noop {} {\emph {\bibinfo {title} {Quantum Machine
				Learning: What Quantum Computing Means to Data Mining}}}\ (\bibinfo
	{publisher} {Academic Press},\ \bibinfo {address} {New York, NY, USA},\
	\bibinfo {year} {2014})\BibitemShut {NoStop}%
	\bibitem [{\citenamefont {Schuld}\ \emph {et~al.}(2015)\citenamefont {Schuld},
		\citenamefont {Sinayskiy},\ and\ \citenamefont
		{Petruccione}}]{Schuld2015introduction}%
	\BibitemOpen
	\bibfield  {author} {\bibinfo {author} {\bibfnamefont {M.}~\bibnamefont
			{Schuld}}, \bibinfo {author} {\bibfnamefont {I.}~\bibnamefont {Sinayskiy}}, \
		and\ \bibinfo {author} {\bibfnamefont {F.}~\bibnamefont {Petruccione}},\
	}\href {\doibase 10.1080/00107514.2014.964942} {\bibfield  {journal}
		{\bibinfo  {journal} {Contemp. Phys.}\ }\textbf {\bibinfo {volume} {56}},\
		\bibinfo {pages} {172} (\bibinfo {year} {2015})}\BibitemShut {NoStop}%
	\bibitem [{\citenamefont {Adcock}\ \emph {et~al.}(2015)\citenamefont {Adcock},
		\citenamefont {Allen}, \citenamefont {Day}, \citenamefont {Frick},
		\citenamefont {Hinchliff}, \citenamefont {Johnson}, \citenamefont
		{Morley-Short}, \citenamefont {Pallister}, \citenamefont {Price},\ and\
		\citenamefont {Stanisic}}]{adcock2015advances}%
	\BibitemOpen
	\bibfield  {author} {\bibinfo {author} {\bibfnamefont {J.}~\bibnamefont
			{Adcock}}, \bibinfo {author} {\bibfnamefont {E.}~\bibnamefont {Allen}},
		\bibinfo {author} {\bibfnamefont {M.}~\bibnamefont {Day}}, \bibinfo {author}
		{\bibfnamefont {S.}~\bibnamefont {Frick}}, \bibinfo {author} {\bibfnamefont
			{J.}~\bibnamefont {Hinchliff}}, \bibinfo {author} {\bibfnamefont
			{M.}~\bibnamefont {Johnson}}, \bibinfo {author} {\bibfnamefont
			{S.}~\bibnamefont {Morley-Short}}, \bibinfo {author} {\bibfnamefont
			{S.}~\bibnamefont {Pallister}}, \bibinfo {author} {\bibfnamefont
			{A.}~\bibnamefont {Price}}, \ and\ \bibinfo {author} {\bibfnamefont
			{S.}~\bibnamefont {Stanisic}},\ }\href@noop {} {\bibfield  {journal}
		{\bibinfo  {journal} {arXiv:1512.02900}\ } (\bibinfo {year}
		{2015})}\BibitemShut {NoStop}%
	\bibitem [{\citenamefont {Biamonte}\ \emph {et~al.}(2016)\citenamefont
		{Biamonte}, \citenamefont {Wittek}, \citenamefont {Pancotti}, \citenamefont
		{Rebentrost}, \citenamefont {Wiebe},\ and\ \citenamefont
		{Lloyd}}]{biamonte2016quantum}%
	\BibitemOpen
	\bibfield  {author} {\bibinfo {author} {\bibfnamefont {J.}~\bibnamefont
			{Biamonte}}, \bibinfo {author} {\bibfnamefont {P.}~\bibnamefont {Wittek}},
		\bibinfo {author} {\bibfnamefont {N.}~\bibnamefont {Pancotti}}, \bibinfo
		{author} {\bibfnamefont {P.}~\bibnamefont {Rebentrost}}, \bibinfo {author}
		{\bibfnamefont {N.}~\bibnamefont {Wiebe}}, \ and\ \bibinfo {author}
		{\bibfnamefont {S.}~\bibnamefont {Lloyd}},\ }\href
	{http://arxiv.org/abs/1611.09347} {\bibfield  {journal} {\bibinfo  {journal}
			{arXiv:1611.09347}\ } (\bibinfo {year} {2016})}\BibitemShut {NoStop}%
	\bibitem [{\citenamefont {Arunachalam}\ and\ \citenamefont
		{de~Wolf}(2017)}]{arunchalam2017survey}%
	\BibitemOpen
	\bibfield  {author} {\bibinfo {author} {\bibfnamefont {S.}~\bibnamefont
			{Arunachalam}}\ and\ \bibinfo {author} {\bibfnamefont {R.}~\bibnamefont
			{de~Wolf}},\ }\href {https://arxiv.org/abs/1701.06806} {\bibfield  {journal}
		{\bibinfo  {journal} {arXiv:1701.06806}\ } (\bibinfo {year}
		{2017})}\BibitemShut {NoStop}%
	\bibitem [{\citenamefont {A\"imeur}\ \emph {et~al.}(2013)\citenamefont
		{A\"imeur}, \citenamefont {Brassard},\ and\ \citenamefont
		{Gambs}}]{aimeur2013quantumspeedup}%
	\BibitemOpen
	\bibfield  {author} {\bibinfo {author} {\bibfnamefont {E.}~\bibnamefont
			{A\"imeur}}, \bibinfo {author} {\bibfnamefont {G.}~\bibnamefont {Brassard}},
		\ and\ \bibinfo {author} {\bibfnamefont {S.}~\bibnamefont {Gambs}},\ }\href
	{\doibase 10.1007/s10994-012-5316-5} {\bibfield  {journal} {\bibinfo
			{journal} {Machine Learning}\ }\textbf {\bibinfo {volume} {90}},\ \bibinfo
		{pages} {261} (\bibinfo {year} {2013})}\BibitemShut {NoStop}%
	\bibitem [{\citenamefont {Rebentrost}\ \emph {et~al.}(2014)\citenamefont
		{Rebentrost}, \citenamefont {Mohseni},\ and\ \citenamefont
		{Lloyd}}]{rebentrost2014quantum}%
	\BibitemOpen
	\bibfield  {author} {\bibinfo {author} {\bibfnamefont {P.}~\bibnamefont
			{Rebentrost}}, \bibinfo {author} {\bibfnamefont {M.}~\bibnamefont {Mohseni}},
		\ and\ \bibinfo {author} {\bibfnamefont {S.}~\bibnamefont {Lloyd}},\ }\href
	{\doibase 10.1103/PhysRevLett.113.130503} {\bibfield  {journal} {\bibinfo
			{journal} {Phys. Rev. Lett.}\ }\textbf {\bibinfo {volume} {113}},\ \bibinfo
		{pages} {130503} (\bibinfo {year} {2014})}\BibitemShut {NoStop}%
	\bibitem [{\citenamefont {Cai}\ \emph {et~al.}(2015)\citenamefont {Cai},
		\citenamefont {Wu}, \citenamefont {Su}, \citenamefont {Chen}, \citenamefont
		{Wang}, \citenamefont {Li}, \citenamefont {Liu}, \citenamefont {Lu},\ and\
		\citenamefont {Pan}}]{cai2015entanglement-based}%
	\BibitemOpen
	\bibfield  {author} {\bibinfo {author} {\bibfnamefont {X.-D.}\ \bibnamefont
			{Cai}}, \bibinfo {author} {\bibfnamefont {D.}~\bibnamefont {Wu}}, \bibinfo
		{author} {\bibfnamefont {Z.-E.}\ \bibnamefont {Su}}, \bibinfo {author}
		{\bibfnamefont {M.-C.}\ \bibnamefont {Chen}}, \bibinfo {author}
		{\bibfnamefont {X.-L.}\ \bibnamefont {Wang}}, \bibinfo {author}
		{\bibfnamefont {L.}~\bibnamefont {Li}}, \bibinfo {author} {\bibfnamefont
			{N.-L.}\ \bibnamefont {Liu}}, \bibinfo {author} {\bibfnamefont {C.-Y.}\
			\bibnamefont {Lu}}, \ and\ \bibinfo {author} {\bibfnamefont {J.-W.}\
			\bibnamefont {Pan}},\ }\href {\doibase 10.1103/physrevlett.114.110504}
	{\bibfield  {journal} {\bibinfo  {journal} {Phys. Rev. Lett.}\ }\textbf
		{\bibinfo {volume} {114}},\ \bibinfo {pages} {110504} (\bibinfo {year}
		{2015})}\BibitemShut {NoStop}%
	\bibitem [{\citenamefont {Lloyd}\ \emph {et~al.}(2016)\citenamefont {Lloyd},
		\citenamefont {Garnerone},\ and\ \citenamefont
		{Zanardi}}]{lloyd2016quantumalgorithms}%
	\BibitemOpen
	\bibfield  {author} {\bibinfo {author} {\bibfnamefont {S.}~\bibnamefont
			{Lloyd}}, \bibinfo {author} {\bibfnamefont {S.}~\bibnamefont {Garnerone}}, \
		and\ \bibinfo {author} {\bibfnamefont {P.}~\bibnamefont {Zanardi}},\ }\href
	{\doibase 10.1038/ncomms10138} {\bibfield  {journal} {\bibinfo  {journal}
			{Nat. Commun.}\ }\textbf {\bibinfo {volume} {7}},\ \bibinfo {pages} {10138}
		(\bibinfo {year} {2016})}\BibitemShut {NoStop}%
	\bibitem [{\citenamefont {Bennett}\ \emph {et~al.}(1999)\citenamefont
		{Bennett}, \citenamefont {DiVincenzo}, \citenamefont {Fuchs}, \citenamefont
		{Mor}, \citenamefont {Rains}, \citenamefont {Shor}, \citenamefont {Smolin},\
		and\ \citenamefont {Wootters}}]{Bennett1999}%
	\BibitemOpen
	\bibfield  {author} {\bibinfo {author} {\bibfnamefont {C.~H.}\ \bibnamefont
			{Bennett}}, \bibinfo {author} {\bibfnamefont {D.~P.}\ \bibnamefont
			{DiVincenzo}}, \bibinfo {author} {\bibfnamefont {C.~a.}\ \bibnamefont
			{Fuchs}}, \bibinfo {author} {\bibfnamefont {T.}~\bibnamefont {Mor}}, \bibinfo
		{author} {\bibfnamefont {E.}~\bibnamefont {Rains}}, \bibinfo {author}
		{\bibfnamefont {P.~W.}\ \bibnamefont {Shor}}, \bibinfo {author}
		{\bibfnamefont {J.~a.}\ \bibnamefont {Smolin}}, \ and\ \bibinfo {author}
		{\bibfnamefont {W.~K.}\ \bibnamefont {Wootters}},\ }\href {\doibase
		10.1103/PhysRevA.59.1070} {\bibfield  {journal} {\bibinfo  {journal} {Phys.
				Rev. A}\ }\textbf {\bibinfo {volume} {59}},\ \bibinfo {pages} {1070}
		(\bibinfo {year} {1999})}\BibitemShut {NoStop}%
	\bibitem [{\citenamefont {Giovannetti}\ \emph {et~al.}(2001)\citenamefont
		{Giovannetti}, \citenamefont {Lloyd},\ and\ \citenamefont
		{Maccone}}]{Giovannetti2001}%
	\BibitemOpen
	\bibfield  {author} {\bibinfo {author} {\bibfnamefont {V.}~\bibnamefont
			{Giovannetti}}, \bibinfo {author} {\bibfnamefont {S.}~\bibnamefont {Lloyd}},
		\ and\ \bibinfo {author} {\bibfnamefont {L.}~\bibnamefont {Maccone}},\ }\href
	{\doibase 10.1038/35086525} {\bibfield  {journal} {\bibinfo  {journal}
			{Nature}\ }\textbf {\bibinfo {volume} {412}},\ \bibinfo {pages} {417}
		(\bibinfo {year} {2001})}\BibitemShut {NoStop}%
	\bibitem [{\citenamefont {Giovannetti}\ \emph {et~al.}(2004)\citenamefont
		{Giovannetti}, \citenamefont {Lloyd},\ and\ \citenamefont
		{Maccone}}]{Giovannetti2004}%
	\BibitemOpen
	\bibfield  {author} {\bibinfo {author} {\bibfnamefont {V.}~\bibnamefont
			{Giovannetti}}, \bibinfo {author} {\bibfnamefont {S.}~\bibnamefont {Lloyd}},
		\ and\ \bibinfo {author} {\bibfnamefont {L.}~\bibnamefont {Maccone}},\ }\href
	{\doibase 10.1126/science.1177170} {\bibfield  {journal} {\bibinfo  {journal}
			{Science}\ }\textbf {\bibinfo {volume} {306}},\ \bibinfo {pages} {1330}
		(\bibinfo {year} {2004})}\BibitemShut {NoStop}%
	\bibitem [{\citenamefont {Niset}\ \emph {et~al.}(2007)\citenamefont {Niset},
		\citenamefont {Ac{\'{i}}n}, \citenamefont {Andersen}, \citenamefont {Cerf},
		\citenamefont {Garc{\'{i}}a-Patr{\'{o}}n}, \citenamefont {Navascu{\'{e}}s},\
		and\ \citenamefont {Sabuncu}}]{Niset2007}%
	\BibitemOpen
	\bibfield  {author} {\bibinfo {author} {\bibfnamefont {J.}~\bibnamefont
			{Niset}}, \bibinfo {author} {\bibfnamefont {A.}~\bibnamefont {Ac{\'{i}}n}},
		\bibinfo {author} {\bibfnamefont {U.~L.}\ \bibnamefont {Andersen}}, \bibinfo
		{author} {\bibfnamefont {N.~J.}\ \bibnamefont {Cerf}}, \bibinfo {author}
		{\bibfnamefont {R.}~\bibnamefont {Garc{\'{i}}a-Patr{\'{o}}n}}, \bibinfo
		{author} {\bibfnamefont {M.}~\bibnamefont {Navascu{\'{e}}s}}, \ and\ \bibinfo
		{author} {\bibfnamefont {M.}~\bibnamefont {Sabuncu}},\ }\href {\doibase
		10.1103/PhysRevLett.98.260404} {\bibfield  {journal} {\bibinfo  {journal}
			{Phys. Rev. Lett.}\ }\textbf {\bibinfo {volume} {98}},\ \bibinfo {pages}
		{260404} (\bibinfo {year} {2007})}\BibitemShut {NoStop}%
	\bibitem [{\citenamefont {Sent{\'{i}}s}\ \emph {et~al.}(2015)\citenamefont
		{Sent{\'{i}}s}, \citenamefont {Gu\c{t}\u{a}},\ and\ \citenamefont
		{Adesso}}]{Sentis2014a}%
	\BibitemOpen
	\bibfield  {author} {\bibinfo {author} {\bibfnamefont {G.}~\bibnamefont
			{Sent{\'{i}}s}}, \bibinfo {author} {\bibfnamefont {M.}~\bibnamefont
			{Gu\c{t}\u{a}}}, \ and\ \bibinfo {author} {\bibfnamefont {G.}~\bibnamefont
			{Adesso}},\ }\href {\doibase 10.1140/epjqt/s40507-015-0030-4} {\bibfield
		{journal} {\bibinfo  {journal} {EPJ Quantum Technology}\ }\textbf {\bibinfo
			{volume} {2}},\ \bibinfo {pages} {17} (\bibinfo {year} {2015})}\BibitemShut
	{NoStop}%
	\bibitem [{\citenamefont {Sent{\'{i}}s}\ \emph {et~al.}(2016)\citenamefont
		{Sent{\'{i}}s}, \citenamefont {Bagan}, \citenamefont {Calsamiglia},
		\citenamefont {Chiribella},\ and\ \citenamefont
		{Mu{\~{n}}oz-Tapia}}]{Sentis2016}%
	\BibitemOpen
	\bibfield  {author} {\bibinfo {author} {\bibfnamefont {G.}~\bibnamefont
			{Sent{\'{i}}s}}, \bibinfo {author} {\bibfnamefont {E.}~\bibnamefont {Bagan}},
		\bibinfo {author} {\bibfnamefont {J.}~\bibnamefont {Calsamiglia}}, \bibinfo
		{author} {\bibfnamefont {G.}~\bibnamefont {Chiribella}}, \ and\ \bibinfo
		{author} {\bibfnamefont {R.}~\bibnamefont {Mu{\~{n}}oz-Tapia}},\ }\href
	{\doibase 10.1103/PhysRevLett.117.150502} {\bibfield  {journal} {\bibinfo
			{journal} {Phys. Rev. Lett.}\ }\textbf {\bibinfo {volume} {117}},\ \bibinfo
		{pages} {150502} (\bibinfo {year} {2016})}\BibitemShut {NoStop}%
	\bibitem [{\citenamefont {Devroye}\ \emph {et~al.}(1996)\citenamefont
		{Devroye}, \citenamefont {Gy\"orfi},\ and\ \citenamefont
		{Lugosi}}]{devroye1996probabilistic}%
	\BibitemOpen
	\bibfield  {author} {\bibinfo {author} {\bibfnamefont {L.}~\bibnamefont
			{Devroye}}, \bibinfo {author} {\bibfnamefont {L.}~\bibnamefont {Gy\"orfi}}, \
		and\ \bibinfo {author} {\bibfnamefont {G.}~\bibnamefont {Lugosi}},\
	}\href@noop {} {\emph {\bibinfo {title} {A Probabilistic Theory of Pattern
				Recognition}}}\ (\bibinfo  {publisher} {Springer},\ \bibinfo {year}
	{1996})\BibitemShut {NoStop}%
	\bibitem [{\citenamefont {Hastie}\ \emph {et~al.}(2008)\citenamefont {Hastie},
		\citenamefont {Tibshirani},\ and\ \citenamefont
		{Friedman}}]{hastie2008statisticallearning}%
	\BibitemOpen
	\bibfield  {author} {\bibinfo {author} {\bibfnamefont {T.}~\bibnamefont
			{Hastie}}, \bibinfo {author} {\bibfnamefont {R.}~\bibnamefont {Tibshirani}},
		\ and\ \bibinfo {author} {\bibfnamefont {J.}~\bibnamefont {Friedman}},\
	}\href@noop {} {\emph {\bibinfo {title} {The Elements of Statistical
				Learning: Data Mining, Inference, and Prediction}}},\ \bibinfo {edition}
	{2nd}\ ed.\ (\bibinfo  {publisher} {Springer},\ \bibinfo {year}
	{2008})\BibitemShut {NoStop}%
	\bibitem [{\citenamefont {Gammerman}\ \emph {et~al.}(1998)\citenamefont
		{Gammerman}, \citenamefont {Vovk},\ and\ \citenamefont
		{Vapnik}}]{gammerman1998learning}%
	\BibitemOpen
	\bibfield  {author} {\bibinfo {author} {\bibfnamefont {A.}~\bibnamefont
			{Gammerman}}, \bibinfo {author} {\bibfnamefont {V.}~\bibnamefont {Vovk}}, \
		and\ \bibinfo {author} {\bibfnamefont {V.}~\bibnamefont {Vapnik}},\ }in\
	\href@noop {} {\emph {\bibinfo {booktitle} {Proceedings of UAI-98, 14th
				Conference on Uncertainty in Artificial Intelligence}}}\ (\bibinfo {year}
	{1998})\ pp.\ \bibinfo {pages} {148--155}\BibitemShut {NoStop}%
	\bibitem [{\citenamefont {Dunjko}\ \emph {et~al.}(2016)\citenamefont {Dunjko},
		\citenamefont {Taylor},\ and\ \citenamefont
		{Briegel}}]{dunjko2016quantumenhanced}%
	\BibitemOpen
	\bibfield  {author} {\bibinfo {author} {\bibfnamefont {V.}~\bibnamefont
			{Dunjko}}, \bibinfo {author} {\bibfnamefont {J.~M.}\ \bibnamefont {Taylor}},
		\ and\ \bibinfo {author} {\bibfnamefont {H.~J.}\ \bibnamefont {Briegel}},\
	}\href {\doibase 10.1103/physrevlett.117.130501} {\bibfield  {journal}
		{\bibinfo  {journal} {Phys. Rev. Lett.}\ }\textbf {\bibinfo {volume} {117}},\
		\bibinfo {pages} {130501} (\bibinfo {year} {2016})}\BibitemShut {NoStop}%
	\bibitem [{\citenamefont {Vapnik}(1995)}]{vapnik1995nature}%
	\BibitemOpen
	\bibfield  {author} {\bibinfo {author} {\bibfnamefont {V.}~\bibnamefont
			{Vapnik}},\ }\href@noop {} {\emph {\bibinfo {title} {The Nature of
				Statistical Learning Theory}}}\ (\bibinfo  {publisher} {Springer},\ \bibinfo
	{address} {New York, NY, USA},\ \bibinfo {year} {1995})\BibitemShut {NoStop}%
	\bibitem [{Note1()}]{Note1}%
	\BibitemOpen
	\bibinfo {note} {See Supplemental Material, which includes Refs.~\cite
		{bengtsson2006geometry,Christandl2007,Christandl2009a}.}\BibitemShut {Stop}%
	\bibitem [{\citenamefont {Bae}\ and\ \citenamefont
		{Ac{\'{i}}n}(2006)}]{Bae2006}%
	\BibitemOpen
	\bibfield  {author} {\bibinfo {author} {\bibfnamefont {J.}~\bibnamefont
			{Bae}}\ and\ \bibinfo {author} {\bibfnamefont {A.}~\bibnamefont
			{Ac{\'{i}}n}},\ }\href {\doibase 10.1103/PhysRevLett.97.030402} {\bibfield
		{journal} {\bibinfo  {journal} {Phys. Rev. Lett.}\ }\textbf {\bibinfo
			{volume} {97}},\ \bibinfo {pages} {030402} (\bibinfo {year}
		{2006})}\BibitemShut {NoStop}%
	\bibitem [{\citenamefont {Sent{\'{i}}s}\ \emph {et~al.}(2010)\citenamefont
		{Sent{\'{i}}s}, \citenamefont {Bagan}, \citenamefont {Calsamiglia},\ and\
		\citenamefont {Mu{\~{n}}oz-Tapia}}]{Sentis2010}%
	\BibitemOpen
	\bibfield  {author} {\bibinfo {author} {\bibfnamefont {G.}~\bibnamefont
			{Sent{\'{i}}s}}, \bibinfo {author} {\bibfnamefont {E.}~\bibnamefont {Bagan}},
		\bibinfo {author} {\bibfnamefont {J.}~\bibnamefont {Calsamiglia}}, \ and\
		\bibinfo {author} {\bibfnamefont {R.}~\bibnamefont {Mu{\~{n}}oz-Tapia}},\
	}\href {\doibase 10.1103/PhysRevA.82.042312} {\bibfield  {journal} {\bibinfo
			{journal} {Phys. Rev. A}\ }\textbf {\bibinfo {volume} {82}},\ \bibinfo
		{pages} {042312} (\bibinfo {year} {2010})}\BibitemShut {NoStop}%
	\bibitem [{\citenamefont {Sent{\'{i}}s}\ \emph {et~al.}(2012)\citenamefont
		{Sent{\'{i}}s}, \citenamefont {Calsamiglia}, \citenamefont
		{Mu{\~{n}}oz-Tapia},\ and\ \citenamefont {Bagan}}]{Sentis2012a}%
	\BibitemOpen
	\bibfield  {author} {\bibinfo {author} {\bibfnamefont {G.}~\bibnamefont
			{Sent{\'{i}}s}}, \bibinfo {author} {\bibfnamefont {J.}~\bibnamefont
			{Calsamiglia}}, \bibinfo {author} {\bibfnamefont {R.}~\bibnamefont
			{Mu{\~{n}}oz-Tapia}}, \ and\ \bibinfo {author} {\bibfnamefont
			{E.}~\bibnamefont {Bagan}},\ }\href {\doibase 10.1038/srep00708} {\bibfield
		{journal} {\bibinfo  {journal} {Sci. Rep.}\ }\textbf {\bibinfo {volume}
			{2}},\ \bibinfo {pages} {708} (\bibinfo {year} {2012})}\BibitemShut {NoStop}%
	\bibitem [{Note2()}]{Note2}%
	\BibitemOpen
	\bibinfo {note} {All these tasks are meant within the framework described in
		Fig. \ref {fig:Q}, and therefore their performance is assessed only through
		registers $\protect \mathcal Y$ and $\protect \mathcal Y'$. In the fully
		quantum case of generating bipartite states $\rho _{XY}$ from $\rho _X$, this
		means that a risk observable would measure the dissimilarity between the
		states $\protect \mathbb Q(\protect \mathrm {tr}_{Y'}\rho _{XY'})$ and
		$\protect \mathrm {tr}_X \rho _{XY'}$, instead of comparing (perhaps more
		naturally in other contexts) the produced bipartite state with a copy of the
		target bipartite state saved as reference.}\BibitemShut {Stop}%
	\bibitem [{\citenamefont {Dieks}(1982)}]{Dieks1982}%
	\BibitemOpen
	\bibfield  {author} {\bibinfo {author} {\bibfnamefont {D.}~\bibnamefont
			{Dieks}},\ }\href {\doibase 10.1016/0375-9601(82)90084-6} {\bibfield
		{journal} {\bibinfo  {journal} {Phys. Lett. A}\ }\textbf {\bibinfo {volume}
			{92}},\ \bibinfo {pages} {271} (\bibinfo {year} {1982})}\BibitemShut
	{NoStop}%
	\bibitem [{\citenamefont {Wootters}\ and\ \citenamefont
		{Zurek}(1982)}]{Wootters1982}%
	\BibitemOpen
	\bibfield  {author} {\bibinfo {author} {\bibfnamefont {W.~K.}\ \bibnamefont
			{Wootters}}\ and\ \bibinfo {author} {\bibfnamefont {W.~H.}\ \bibnamefont
			{Zurek}},\ }\href {\doibase 10.1038/299802a0} {\bibfield  {journal} {\bibinfo
			{journal} {Nature}\ }\textbf {\bibinfo {volume} {299}},\ \bibinfo {pages}
		{802} (\bibinfo {year} {1982})}\BibitemShut {NoStop}%
	\bibitem [{Note3()}]{Note3}%
	\BibitemOpen
	\bibinfo {note} {We note that a de Finetti theorem for fully symmetric
		quantum channels can be found in the literature~\cite
		{fawzi_quantum_2015}.}\BibitemShut {Stop}%
	\bibitem [{\citenamefont {Christandl}\ \emph {et~al.}(2007)\citenamefont
		{Christandl}, \citenamefont {K{\"{o}}nig}, \citenamefont {Mitchison},\ and\
		\citenamefont {Renner}}]{Christandl2007}%
	\BibitemOpen
	\bibfield  {author} {\bibinfo {author} {\bibfnamefont {M.}~\bibnamefont
			{Christandl}}, \bibinfo {author} {\bibfnamefont {R.}~\bibnamefont
			{K{\"{o}}nig}}, \bibinfo {author} {\bibfnamefont {G.}~\bibnamefont
			{Mitchison}}, \ and\ \bibinfo {author} {\bibfnamefont {R.}~\bibnamefont
			{Renner}},\ }\href {\doibase 10.1007/s00220-007-0189-3} {\bibfield  {journal}
		{\bibinfo  {journal} {Commun. Math. Phys.}\ }\textbf {\bibinfo {volume}
			{273}},\ \bibinfo {pages} {473} (\bibinfo {year} {2007})}\BibitemShut
	{NoStop}%
	\bibitem [{\citenamefont {Sasaki}\ and\ \citenamefont
		{Carlini}(2002)}]{Sasaki2002}%
	\BibitemOpen
	\bibfield  {author} {\bibinfo {author} {\bibfnamefont {M.}~\bibnamefont
			{Sasaki}}\ and\ \bibinfo {author} {\bibfnamefont {A.}~\bibnamefont
			{Carlini}},\ }\href {\doibase 10.1103/PhysRevA.66.022303} {\bibfield
		{journal} {\bibinfo  {journal} {Phys. Rev. A}\ }\textbf {\bibinfo {volume}
			{66}},\ \bibinfo {pages} {022303} (\bibinfo {year} {2002})}\BibitemShut
	{NoStop}%
	\bibitem [{\citenamefont {Bisio}\ \emph {et~al.}(2010)\citenamefont {Bisio},
		\citenamefont {Chiribella}, \citenamefont {D'Ariano}, \citenamefont
		{Facchini},\ and\ \citenamefont {Perinotti}}]{Bisio2010}%
	\BibitemOpen
	\bibfield  {author} {\bibinfo {author} {\bibfnamefont {A.}~\bibnamefont
			{Bisio}}, \bibinfo {author} {\bibfnamefont {G.}~\bibnamefont {Chiribella}},
		\bibinfo {author} {\bibfnamefont {G.~M.}\ \bibnamefont {D'Ariano}}, \bibinfo
		{author} {\bibfnamefont {S.}~\bibnamefont {Facchini}}, \ and\ \bibinfo
		{author} {\bibfnamefont {P.}~\bibnamefont {Perinotti}},\ }\href {\doibase
		10.1103/PhysRevA.81.032324} {\bibfield  {journal} {\bibinfo  {journal} {Phys.
				Rev. A}\ }\textbf {\bibinfo {volume} {81}},\ \bibinfo {pages} {032324}
		(\bibinfo {year} {2010})}\BibitemShut {NoStop}%
	\bibitem [{\citenamefont {Nakahira}\ \emph {et~al.}(2015)\citenamefont
		{Nakahira}, \citenamefont {Kato},\ and\ \citenamefont
		{Usuda}}]{Nakahira2015}%
	\BibitemOpen
	\bibfield  {author} {\bibinfo {author} {\bibfnamefont {K.}~\bibnamefont
			{Nakahira}}, \bibinfo {author} {\bibfnamefont {K.}~\bibnamefont {Kato}}, \
		and\ \bibinfo {author} {\bibfnamefont {T.~S.}\ \bibnamefont {Usuda}},\ }\href
	{\doibase 10.1103/PhysRevA.91.052304} {\bibfield  {journal} {\bibinfo
			{journal} {Phys. Rev. A}\ }\textbf {\bibinfo {volume} {91}},\ \bibinfo
		{pages} {052304} (\bibinfo {year} {2015})}\BibitemShut {NoStop}%
	\bibitem [{\citenamefont {Chiribella}\ \emph {et~al.}(2004)\citenamefont
		{Chiribella}, \citenamefont {D'Ariano}, \citenamefont {Perinotti},\ and\
		\citenamefont {Sacchi}}]{Chiribella2004}%
	\BibitemOpen
	\bibfield  {author} {\bibinfo {author} {\bibfnamefont {G.}~\bibnamefont
			{Chiribella}}, \bibinfo {author} {\bibfnamefont {G.~M.}\ \bibnamefont
			{D'Ariano}}, \bibinfo {author} {\bibfnamefont {P.}~\bibnamefont {Perinotti}},
		\ and\ \bibinfo {author} {\bibfnamefont {M.~F.}\ \bibnamefont {Sacchi}},\
	}\href {\doibase 10.1103/PhysRevA.70.062105} {\bibfield  {journal} {\bibinfo
			{journal} {Phys. Rev. A}\ }\textbf {\bibinfo {volume} {70}},\ \bibinfo
		{pages} {062105} (\bibinfo {year} {2004})}\BibitemShut {NoStop}%
	\bibitem [{\citenamefont {Chiribella}\ \emph {et~al.}(2008)\citenamefont
		{Chiribella}, \citenamefont {D'Ariano},\ and\ \citenamefont
		{Perinotti}}]{Chiribella2008}%
	\BibitemOpen
	\bibfield  {author} {\bibinfo {author} {\bibfnamefont {G.}~\bibnamefont
			{Chiribella}}, \bibinfo {author} {\bibfnamefont {G.~M.}\ \bibnamefont
			{D'Ariano}}, \ and\ \bibinfo {author} {\bibfnamefont {P.}~\bibnamefont
			{Perinotti}},\ }\href {\doibase 10.1103/PhysRevLett.101.060401} {\bibfield
		{journal} {\bibinfo  {journal} {Phys. Rev. Lett.}\ }\textbf {\bibinfo
			{volume} {101}},\ \bibinfo {pages} {060401} (\bibinfo {year}
		{2008})}\BibitemShut {NoStop}%
	\bibitem [{\citenamefont {Bengtsson}\ and\ \citenamefont
		{\.Zyczkowski}(2006)}]{bengtsson2006geometry}%
	\BibitemOpen
	\bibfield  {author} {\bibinfo {author} {\bibfnamefont {I.}~\bibnamefont
			{Bengtsson}}\ and\ \bibinfo {author} {\bibfnamefont {K.}~\bibnamefont
			{\.Zyczkowski}},\ }\href {http://chaos.if.uj.edu.pl/~karol/geometry.htm}
	{\emph {\bibinfo {title} {{G}eometry of quantum states: an introduction to
				quantum entanglement}}}\ (\bibinfo  {publisher} {Cambridge University
		Press},\ \bibinfo {year} {2006})\BibitemShut {NoStop}%
	\bibitem [{\citenamefont {Christandl}\ and\ \citenamefont
		{Toner}(2009)}]{Christandl2009a}%
	\BibitemOpen
	\bibfield  {author} {\bibinfo {author} {\bibfnamefont {M.}~\bibnamefont
			{Christandl}}\ and\ \bibinfo {author} {\bibfnamefont {B.}~\bibnamefont
			{Toner}},\ }\href {\doibase 10.1063/1.3114986} {\bibfield  {journal}
		{\bibinfo  {journal} {Journal of Mathematical Physics}\ }\textbf {\bibinfo
			{volume} {50}},\ \bibinfo {pages} {042104} (\bibinfo {year}
		{2009})}\BibitemShut {NoStop}%
	\bibitem [{\citenamefont {Fawzi}\ and\ \citenamefont
		{Renner}(2015)}]{fawzi_quantum_2015}%
	\BibitemOpen
	\bibfield  {author} {\bibinfo {author} {\bibfnamefont {O.}~\bibnamefont
			{Fawzi}}\ and\ \bibinfo {author} {\bibfnamefont {R.}~\bibnamefont {Renner}},\
	}\href {\doibase 10.1007/s00220-015-2466-x} {\bibfield  {journal} {\bibinfo
			{journal} {Commun. Math. Phys.}\ }\textbf {\bibinfo {volume} {340}},\
		\bibinfo {pages} {575} (\bibinfo {year} {2015})}\BibitemShut {NoStop}%
\end{thebibliography}
\end{document}